\documentclass{article}
\usepackage{authblk}
\usepackage{xspace}
\usepackage[utf8]{inputenc}
\usepackage{amsmath,amssymb}
\usepackage[utf8]{inputenc}
\usepackage[english]{babel}
\usepackage[toc,page]{appendix}
\usepackage[colorlinks=TRUE, linkcolor=blue, urlcolor=blue, citecolor=blue]{hyperref}
\usepackage[paper=portrait,pagesize]{typearea}
\usepackage[nocenter]{qtree}
\usepackage{tree-dvips}
\usepackage{tikz}
\usetikzlibrary{arrows}
\usepackage{breqn}
\usepackage{multirow}
\usepackage{floatrow}
\usepackage{amsthm}
\usepackage{subfig}
\usepackage{lipsum}
\usepackage{amsmath}
\usepackage{float}
\usepackage[toc,page]{appendix}
\usepackage[margin = 2.5cm]{geometry}     
\usepackage{booktabs}                     
\usepackage[per-mode = symbol]{siunitx}   

\usepackage[colorlinks=TRUE, linkcolor=blue, urlcolor=blue, citecolor=blue]{hyperref}

\usepackage{tikz}
\newtheorem{definition}{Definition}[section]

\newtheorem{theorem}{Theorem}[section]
\newtheorem{lemma}[theorem]{Lemma}
\usetikzlibrary{bayesnet}

\title{Use of High Dimensional Modeling for automatic variables selection in econometric linear model, the \textit{best path algorithm} }

\author{Luigi Riso $^1$ }
\date{%
    $^1$ Università  Cattolica del Sacro Cuore,\\ luigi.riso@unicatt.it\\
    \today
}

\begin{document}
\flushbottom
\maketitle

 \begin{abstract}

This paper proposes a new algorithm for an automatic variable selection procedure in econometric linear models. Using High Dimensional Graphical Models, it is shown how the variables that are relevant for a specific variable of interest can be properly identified. 
The key of this selection procedure is mutual information which is used to measure dependencies between random variables. Several contributions in literature have investigated the use of mutual information in selecting the appropriate number of relevant features in a large data-set, but most of them have focused on binary outcomes or required high computational effort. The algorithm here proposed overcomes these drawbacks and can be also applied in a high dimensional contest insofar as it is an extension of  Chow and Liu's algorithm which was devised for explaining the relationships between a large number of the variables.  
\\
\\
\textbf{Keywords}: {Graphical Models, Chow-Liu Algorithm, Automatic Feature Selection, Mutual Information, Econometric linear models}
\end{abstract}
\section{Introduction}
The vast availability of large datasets might create problems to researcher in choosing the  explicative variables of an econometric model. In this case, the use of dimensional reduction, \cite{altman2018curse, eklund2007embarrassment}, becomes imperative to overcome drawbacks such as endogeneity of the covariates, degrees of freedom and collinearity.
In general, a dimensionality reduction tecnique can be expressed as an optimization problem, over a dataset $\textbf{X}$  \cite{saxena2008dimensionality}. In details, let $\textbf{X}$ be a $n \times p$ matrix collecting $n$ observations on  \begin{math} p\end{math} variables \begin{math} \textbf{x}_i\end{math}. 
The purpose of dimensional reduction is to obtain a new dataset \begin{math} \textbf{Y} \end{math}
 with \begin{math} {k} \end{math} variables instead of  the original \begin{math} {p} \end{math} ones, that is a dataset 
 with dimensionality \begin{math} {k} \end{math}, where \begin{math} k < p \end{math} and often \begin{math} k\ll p \end{math}.\\ There are two basic strategies used to reduce the dimensions from $p$ to $k$: feature selection and feature extraction.
Feature selection techniques focus on selecting some of the most important features from a dataset, 
while feature extraction techniques generate new features which explain either some or the entire dataset information . From an econometric point of view, feature selection (or variable selection) is more attractive than the feature extraction, because it allows to carry out a more efficient and accurate analysis by eliciting irrelevant information included in a large set of variables \cite{uematsu2019high}.\\
The novelty of this paper is to propose an automatic variable selection algorithm which makes use of Probabilistic Graphical Models  \cite{jordan2004graphical}, to detect the relevant explicative variables 
of a linear econometric model on the basis of a large dataset. The Probabilistic Graphical Model, which is used to implement the algorithm, is an extension of the \textit{Chow-Liu Algorithm} proposed  by \cite{edwards2010selecting} to find the maximum likelihood tree of mixed dataset composed of both discrete and continuous variables. As it is well known, trees and forests are special cases of undirected graphs. A forest is an acyclic undirected graph, that is, an undirected graph with no cycles. Connected components of a forest are called trees \cite{hojsgaard2012graphical}. The nodes in a tree identify the random variables of the dataset, while connections between nodes represent joint probability distributions \cite{chow1968approximating} which are determined via \textit{mutual information}.\\
In this paper, the notion of mutual information is used to select the relevant variables of an econometric model. In literature some papers  investigated the use of mutual information to select an appropriate number of relevant features \cite{battiti1994using, estevez2009normalized, kwak2002input} from a dataset, but most of these algorithms focus on binary variables and their implementation requires high computational effort \cite{kwak2002input}. One of the advantages of the algorithm here proposed is that it works well also in a high dimensional contest as it is computationally efficient in high dimensional contexts as molecular biology and other fields where are required models with hundreds to tens of thousands of variables.
 \cite{edwards2010selecting, meila1999accelerated, kirshner2012conditional}. Moreover, unlike the machine learning approach, it produces a final model whose goodness can be evaluated from an econometric point of view via fit criteria. Thus, the resulting final model, besides allowing an efficient analysis, can be properly evaluated by using statistical measures \cite{stock2002forecasting}.\\ The remainder of this paper is organized as follows. In Section \ref{2}, we recall the basic properties of  the extension of the \textit{Chow-Liu Algorithm}, in Section \ref{3}, we present the algorithm and the theory behind it, finally Section \ref{4} shows an illustrative example on real data and conclusions follow in Section \ref{5}. 

\section{Graphical and High-Dimensional Graphical Models}
\label{2}
In this section we introduce the essential of graphical and high-dimensional graphical models which prove useful in devising an algorithm to select variables in a large dataset which is dealt with in the next section.\\
Graphical Models are used to specify the conditional independence relationships between random variables of a dataset. Graphically, these  relationships are depicted as a networks of variables in a graph. 
A graph is a mathematical object \begin{math}
\textbf{G}=(V, E)
\end{math}, where $V$ is a finite set of nodes in a one-to-one correspondence with the random variables of the dataset, and \begin{math}
E \subset V \times V 
\end{math}, is a subset of ordered couples of $V$,  that defines edges or links representing the interactions between nodes \cite{jordan2004graphical}. Two generic node $u$ and $v$ in a graph $\textbf{G}=(V,E)$, are connected if there is a sequence $u=v_1,..,v_k=v$ of distinct nodes such that $(v_{i-1}, v_i) \in E $, $\forall i=1,...,k $. The sequence $u=v_1,..,v_k=v$ is called \textit{path} \cite{de2009high}.
\\ In this analysis we make use of \textit{High-Dimensional Graphical Models}, which prove useful to represent the relationships between a large set of variables. Indeed, we consider a dataset composed by \begin{math}
\textit{n}
\end{math} observations on \textit{p} random variables  \begin{math}
\textbf{X}_p
\end{math} that are collected in a   \begin{math} 
N \times  p
\end{math} matrix \begin{math}
\textbf{X}
\end{math}. 
We assume that these \begin{math}  
p
\end{math} variables are split into two sets: a set of \textit{d} discrete  \begin{math} \textbf{Z}=(\textit{Z}_1,...,\textit{Z}_d) \end{math} and a set of \textit{q} continuous  \begin{math}
\textbf{Y}=(\textit{Y}_1,...,\textit{Y}_q)
\end{math}  variables.
Accordingly, the \textit{i}-observation of the dataset \begin{math}
\textbf{X}=(\textbf{Z},\textbf{Y})\end{math} can be expressed as \begin{math}
(z_i,y_i).\end{math} with $z_i$ and $y_i$ denoting the $i_th$ observation of the variables $\textit{Z}_i$ and $\textit{Y}_i$, respectively.\\
Given the one-to-one correspondence between variables and nodes, we can write the sets of nodes \begin{math}
V
\end{math} as \begin{math}
V=\{ \Delta \cup \Gamma \}
\end{math} where \begin{math}
\Delta
\end{math} and \begin{math}
\Gamma
\end{math} are the nodes corresponding to the variables in \begin{math} \textbf{Z}\end{math} and \begin{math} \textbf{Y}\end{math}, respectively. \\
In the following, we restrict our attention to discrete random variables,  and denote with $z=(z_1,...,z_d)$ the generic observation (or cell) of $\textbf{Z}$. In this case the set, $\mathcal{Z}$, of the possible cells of the variables $\textbf{Z}$, also called \textit{levels}, may be labelled as  $1,...,|Z_v|$. Now, we assume that the cell probabilities factorize according to a an unknown tree $\tau$  as  follows $\textbf{G}_Z=(\Delta, E_{\Delta})$, where $\Delta$ and $E_{\Delta}$ are the vertex and the edges set, respectively.  Accordingly, the cell probabilities can be written as follows
\begin{equation}
p(z)=\prod_{e \in E_{\Delta}} g_e(z) 
\end{equation}
for a given function $g_e(z)$ that depends on the variables included in the edges set $e$.\\ Should $e=(Z_u,Z_v)$, then $g_e(z)$ would be only a function of $z_u$ and $z_v$ \cite{edwards2010selecting} and the cell probabilities would take the form (\cite{chow1968approximating}) 

\begin{equation}\label{eq:l1}
    p(\textbf{z}| \tau)= \frac{\prod_{u,v \in E_{\Delta}} p(z_u,z_u) }
    {\prod_{v \in V} p (z_v)^{d_v-1}}=
    \prod_{v \in V} p (z_v)  \prod_{u,v \in E_{\Delta}} \frac{p(z_u,z_u)}{p(z_u)p(z_v)}
\end{equation}
where $d_v$ is the degree of $v$, that is the number of edges incident to the node $v$.\\ In light of \eqref{eq:l1}, the maximized log-likelihood, up to a constant, turns out to be
\begin{equation}
\sum _{(u,v) \in E_{\Delta}} I_{u,v}
\end{equation}
The quantity $I_{u,v}$, called \textit{mutual information}, is defined as follows
\begin{equation*}
    I_{u,v}= \sum_{z_u,z_v} n(z_u,z_v) \ln \frac{n(z_u,z_v)}{n(z_u) n(z_v)}
\end{equation*}
 where $n(z_u,z_v)$ is the number of observations with $Z_u=z_u$ and $Z_v=z_v$. \\
\cite{lewis1959approximating} defined the mutual information between two variables as a measure of their closeness.\\ 
The maximum likelihood tree for the entire set
\textbf{\textit{Z}} 
of the\textit{d} discrete random variables can be obtained by computing 
$I_{u,v}$ as edges weights on the complete graph with vertex set $\Delta$ using a maximum spanning tree algorithm \cite{chow1968approximating} . \\
It can be proved that $I_{u,v}$ is one half the usual statistic $G$ of the likelihood ratio test for marginal independence of $Z_u$ from $Z_v$, that is
 \begin{equation}\label{eq:test}
 2 I{u,v}=G 
 \end{equation}
 which is calculated by using the table of count $\{n(z_u,z_v)\}$ built by cross-tabulating $Z_u$ and $Z_v$. 
 Under the marginal independence $G$ has an asymptotic $\chi^2_{(k)}$ distribution, with k degrees of freedom, where  $k=(|X_u|-1)(|X_v|-1)$  
is the number of parameters tested under the null \cite{edwards2012introduction}.\\
 Following a specular approach, we can obtain the maximum likelihood tree $\textbf{G}_Y=(\Gamma, E_{\Gamma})$ for a set $\textbf{Y}$ of continuous random variables that can be assumed to have a multivariate Gaussian distribution. Here the sample mutual information between two margins $Y_u$ and $Y_v$ is given by
 \begin{equation*}
     I_{u,v}=-N\ \frac{ln (1-\hat{\rho}^{2}_{u,v})}{2}
 \end{equation*}
where $\hat{\rho}_{u,v}$ is the sample correlation between $Y_v$ and $Y_u$. As before, the mutual information is related to the likelihood ratio test statistic as in \eqref{eq:test}.
Under marginal independence between $Y_u$ and $Y_v$, $G$  has a $\chi^2_{(1)}$ distribution \cite{edwards2012introduction} \\
The \cite{chow1968approximating} algorithms, which are finalized to find the maximum weight spanning tree
of a arbitrary undirected connected graph with positive
edge weights, have been studied thoroughly . 
In this regard, the algorithm  due to \cite{kruskal1956shortest} provides a simple and efficient solution to this problem. Starting with a null graph, it proceeds by adding at each step the edge with the largest
weight that does not form a cycle with the ones already chosen.
 \cite{edwards2010selecting} proposed an extension of the Chow-Liu Algorithm that can be applied with mixed dataset $\textbf{X}$. This algorithm relies on the use of mutual information between a discrete variable $Z_u$ and a continuous variable $Y_v$ and it is characterized by the fact that the marginal model is a simple
ANOVA model (section 4.1.7 \cite{edwards2012introduction}).\\
 In order to find the mutual information $I(z_u,y_v)$ between each couple of variables when dealing with mixed variables, it is important to distinguish the case when the variance of $Y_v$ is distributed homogeneously across the levels of the discrete variable $Z_u$ from that when it is heterogeneously distributed \cite{edwards2012introduction}.
In the homogenous case,  if we denote with $
\{ n_i, \bar{y}_v, s_i^{(v)}\}_{i=1,...,|Z_u|}$ the sample cell count, mean and  variance of the couple of variables 
$(Z_u, Y_v)$, an estimator of mutual information  is given by:
\begin{equation*}
    I(z_u,y_v)=\frac{N}{2} \log \left(\frac{s_0}{s} \right),
\end{equation*}
where  
$s_0=\sum_{k=1}^N(y_v^{(k)}-\hat{y}_v)/N,
\quad
s=\sum_{i=1}^{|Z_u|}n_is_i/N
$, and $
k_{z_u,y_v}= |Z_u|-1
$ are the degrees of freedom of the test for marginal independence between the discrete variable $Z_u$ and the continuous variable $Y_v$.\\
In the heterogeneous case an estimator of the mutual information is given by \begin{equation*}
    I(z_u,y_v)=\frac{N}{2} \log (s_0) -\frac{1}{2} \sum_{i=1,...,|Z_s|} n_i \log (s_i)
\end{equation*}
where $
k_{z_u,y_v}= 2(|Z_u|-1) 
$ are degrees of freedom of the test for the marginal independence between the discrete variable $Z_u$ and continuous variable $Y_v$, with statistic specified as in \eqref{eq:test} and with
a $\chi^2_{(k_{u,v})}$ distribution.\\
 \cite{edwards2010selecting} suggested also the use of one of the following measures to avoid the inclusion of links not supported by data
 \begin{equation}
I^{AIC}=I(x_i,x_j)-2k_{x_i,x_j} \end{equation} or
\begin{equation}
I^{BIC}=I(x_i,x_j)-\log(n) k_{x_i,x_j}
\end{equation}
where \begin{math}
k_{x_i,x_j}
\end{math} are the degrees of freedom associated with the pair of variables, that are defined according to the nature of the variables involved.\\
The above measures are employed in an algorithm to find the best-spanning tree. The algorithm stops once the maximum number of edges has been included in the graph.\\ 
Finally, it is worth remembering that High-Dimensional Graphical Models are strong decomposable (section 7.4 of \cite{hojsgaard2012graphical}). Strongly decomposabilty is a useful property of graphical models  \cite{lauritzen1989graphical} as it allows to restrict the analysis to sub-regions or sub-models of the graph which are local strongly decomposable models with minimal AIC/BIC \cite{lauritzen1989graphical, edwards2010selecting}.\\
A mixed graphical model is strongly decomposable if and only if it is acyclic and contains no \textit{forbidden paths} \cite{lauritzen1989graphical}. A forbidden path is a path between two non-adjacent discrete vertices passing through continuous vertices (more details \cite{lauritzen1996graphical}, p 7-12). Now, High-Dimensional Graphical Models are strong decomposable as, on the one hand, trees and forests are acyclic and, on the other hand, the \cite{edwards2010selecting} algorithm avoids the presence of forbidden paths.  

\section{The Best Path Algorithm }
\label{3}
To introduce the algorithm, let us consider a dataset $\textbf{D}$ composed of  $n$ observations on $p$ variables, with $ p\ll n$. Accordingly, a generic $i-$ observation of the dataset is a vector including the $i$-th observation on the random variables $\{ Y_{i},X_{i,1},...,X_{i,p-1}\}$. The best path algorithm requires the construction of a High-Dimensional Graphical Model $G=(V,E)$ for the dataset $\textbf{D}$. The High-Dimensional Graphical Model is built  to select among the variables $\{X_{1},...,X_{p-1}\}$ the minimal subset that is  really effective to explain $Y$. To this end, we introduce the following preliminary definitions

\begin{definition}[Distance]
\label{distance}
Let us assume that a path  exists between the nodes $Y$ and $X_j$, that is
\begin{equation}
Y=X_{j+1},...,X_{j+k}=X_j
\end{equation}
Then, we define distance of order $k$, $d_{Y,X_j}^{(k)}$, the number $k$ of the nodes present in the path from $Y$ to $X_j$. It is simple to note that:
\begin{equation}
    d_{Y,X_j}^{(k)}=n_{(X_{j+1},...,X_{j+k})}=d_{X_j,Y}^{(k)}
\end{equation}
where $n_{(X_{j+1},...,X_{j+k})}$ is the number of the nodes included in the path.
\end{definition}
Note that, according to Definition \ref{distance}, it is not possible to compute the distance  between nodes belonging to different trees.

\begin{definition}[Path Steps]
\label{path-s}
A path step $w_k$, with $k=1,..., max\{d_{Y,X_j}\}$,
for the node $Y$, is the sub set of variables $X_i$ with distance from $Y$ is equal or lower than $k$, that is 
\begin{equation}
  \textbf{X}_{w_k}= \{d_{Y,X_i} \leq k\}
\end{equation}
The path step $w_k$ corresponds to the sub-regions of the graph $\textbf{G}=(V,E)$ that satisfies the following property
\begin{equation}
   w_k=\textbf{G}_k=(V_{d_{Y,X_i} \leq k},E_{d_{Y,X_i} \leq k})
\end{equation}
\end{definition}
In the following we provide an example of the above definitions \ref{distance} and \ref{path-s}. \\
Let us consider graph \ref{fig1} and assume that 
$Y$ is the node of interest. In order to detect which subset of the $X_j$ variables are effective in explaining $Y$, we operate an initial selection  by considering only those variables that  
compose the tree to which $Y$ belongs. Indeed, as mutual information between $Y$ and the variables non included in its tree can be presumed to be approximately null, the  hypothesis of marginal independence between $Y$ and these variables can be assumed to hold.
Thus, given the structure of graph \ref{fig1}, the following holds
\begin{equation}
I(Y, X_{13}) \approx I(Y, X_{12})\approx I(Y, X_{11}) \approx 0
\end{equation}
and, accordingly, we can assume the hypothesis of marginal independence between $Y$ and $X_j$, $j=12,13,14$ \cite{edwards2012introduction}. Consequently, the latter variables can be discarded from the process of feature selection.

\begin{figure}[H]\par\medskip
\centering
\includegraphics[scale=0.60]{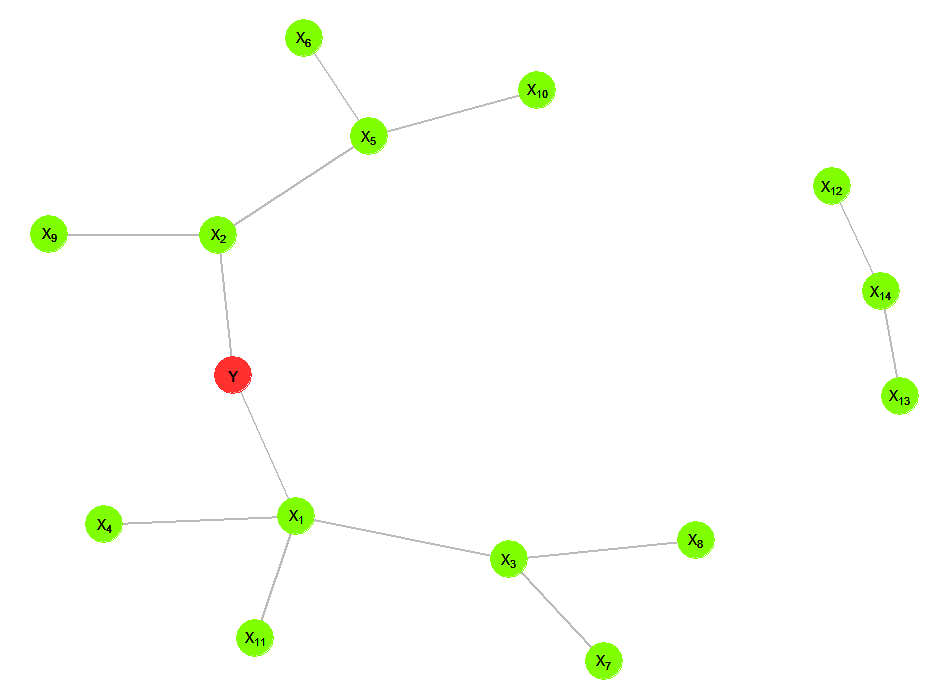}
\caption{Example min BIC forest, from a dataset $\textbf{D}$ with $p=15$ variables and $n$ observation }
\label{fig1}
\end{figure}

Now, let us consider the distances of increasing order between the node $Y$ and the variables $X_{j}$, $j=1,...11$ belonging to its same tree
\begin{equation}
\label{dis}
\centering
\begin{split}
  d_{Y,X_1}^{(1)}=d_{Y,X_2}^{(1)} ;\\
  d_{Y,X_9}^{(2)}=d_{Y,X_4}^{(2)}=d_{Y,X_{11}}^{(2)}=d_{Y,X_5}^{(2)}=d_{Y,X_3}^{(2)} ;\\
  d_{Y,X_7}^{(3)}=d_{Y,X_6}^{(3)}=d_{Y,X_{10}}^{(3)}=d_{Y,X_8}^{(3)}
    \end{split}
    \end{equation}
These distances can be used to compute the \textit{path steps} associated to the node $Y$
\begin{equation}
    \begin{split}
        w_1=\{ X_1, X_2 \};\\
        w_2=\{ X_1, X_2, X_3, X_4, X_5, X_9, X_{11} \}; \\
        w_3=\{ X_1, X_2, X_3, X_4, X_5, X_9, X_{11}, X_6, X_7, X_8, X_{10} \}
    \end{split}
\end{equation}
which include the potential subsets of variables to consider in the feature selection for $Y$.\\ The following theorem \ref{T} explains how to determine the best \textit{path step} which contains the variables that are potentially relevant for the variable of interest.

\begin{theorem}[The best path step ]
\label{T}
In High-Dimensional Graphical Models $\textbf{G}=(V,E)$ a path step $w_i$  that satisfies the following property 
\begin{equation}\label{eq:tbp}
    \sum_{\textbf{X} \in w_i} I(Y, \textbf{X})  \geq \max \left\{   \sum_{\textbf{X} \in w_j} I(Y, \textbf{X}) \right\},   i \ne j, j=1,..,k
\end{equation}
always exists for the node of interest $Y$. The path step $w_i$ is the best path step, insofar as it includes the variables that maximize the mutual information with $Y$
\end{theorem}
\begin{proof}
To prove the existence of a path step satisfying \eqref{eq:tbp}, let us consider the sum of the mutual information between the variable $Y$ and the set of the variables $\{X_1,...,X_h\}$ belonging to its same tree 
\begin{equation} \label{eq:bp}
    I(Y, \textbf{X})= I(Y,X_1)+I(Y,X_2)+ \dots+ I(Y,X_h) 
\end{equation}
Formula \eqref{eq:bp} can be rewritten as follows
\begin{equation}
\label{8}
    I(Y, \textbf{X}) = \sum_{d_{Y,X}^{(1)}} I(Y,\textbf{X})+\sum_{d_{Y,X}^{(2)}} I(Y,\textbf{X})+\dots +\sum_{max\{d_{Y,X}\}}I(Y,\textbf{X})
\end{equation}
where $\sum_{d_{Y,X}^{(k)}}$ is the sum of the mutual information between $Y$ and the variables $X_j$ which are $k=1,2...$ nodes distant from $Y$ (see definition \ref{distance}).\\
According to definition \ref{path-s}, the above formula can be written in terms of variables  belonging to the paths step $w_k$, $k=1,2...$.
According to Definition 3.2, we can approximate the mutual information between $Y$ and all the variables $\textbf{X}$ as follows:
\begin{equation}
\label{9}
    I(Y, \textbf{X}) = \sum_{\textbf{X} \in {max(w)}}I(Y,\textbf{X})
\end{equation}
Now, as mutual information is a non-negative measure satisfying the following property
\begin{equation}
\sum_{\textbf{X} \in w_1 }I(Y,\textbf{X}) \leq \sum_{\textbf{X} \in w_2 }I(Y,\textbf{X}) \leq \dots \leq \sum_{\textbf{X} \in {max(w)}}I(Y,\textbf{X})
\end{equation}
the $i$-th path step $w_i$ turns out to be the best one if the mutual information between $Y$ and the variables included in the $i$-th path  step, is significant higher than the mutual information included in lower path steps ($w_j, j<i$), but approximately equal to that included in higher path steps ($w_j, j>i$), that is
\begin{equation}
\sum_{w_{i-1}} I(Y,\textbf{X}) < \sum_{w_{i}} I(Y,\textbf{X})\approx \sum_{w_{i+1}} I(Y,\textbf{X})\approx...\approx \sum_{\textbf{X} \in {max(w)}}I(Y,\textbf{X})
\end{equation}
\end{proof}
\begin{lemma}
After proving the presence of a best path step in every High Dimensional graphical model and explaining the property of the best path in terms of mutual information, let us see how this notion can be properly connected to statistical measure employed to assess the goodness of fit of linear econometric model. To this end, it is important to first establish a relationship between mutual information and  \textit{Adjusted R-squared} $\bar{R}^2$
\begin{equation}
\label{R2}
    \bar{R}^2=1-\frac{\sum_{j=1}^n(Y_j-\hat{f}(X_j))^2}{n-o-1}\frac{n-1}{\sum_{j=1}^n(Y_j-\bar{y})^2}
\end{equation}
where $\hat{f}(X_j)$ is the estimated value for the $j$-th observation of $Y$ from a model explaining this variable, $\bar{y}$ the sample mean of the latter, $o$ the number of explanatory variables and $n$ the number of observations. As is well known, $\bar{R}^2$ allows 
to quantify the extent to which regressors account for,  or explain, variation of the dependent variable \cite{stock2012introduction}. \\
Now, consider two sets of variables $G_{i}=(X_1,...X_{j})$ and $G_{i+1}=(X_{1},..,X_{j}, X_{q})$ and assume that 
\begin{equation}\label{eq:lk}
    \mathcal{L}(Y|X_1,...X_{q})=\mathcal{L}(Y|X_1,...X_{j})
\end{equation}
where $\mathcal{L}(Y|X_1,...X_{q})$ stands for the conditional distribution of $Y$ relative to $X_1,...X_{q}$. According to \eqref{eq:lk}, $X_1,...X_{j}$ can be regarded as an effective set of variables to explain $Y$, in other words the set $(X_{j+1},.., X_{q})$ doesn't add information to explain the outcome $Y$ \cite{causeur2003linear}. It follows that
\begin{equation}
\sum_{X \in G_{i}}I(Y, X) \approx \sum_{X \in G_{i+1}}I(Y, X) 
\end{equation}
Therefore, if $w_i$ denotes the best path step including the set of variables $G_{i}$ which is relevant for $Y$, as specified in \eqref{eq:lk}, the following must hold
\begin{equation}\label{eq:dis}
    \bar{R}^2_{w_{i-1}} \leq \bar{R}^2_{{inter}} \leq \bar{R}^2_{w_{i}} 
\end{equation}
 where $\bar{R}^2_{{inter}}$ is the $\bar{R}^2$ of a model that includes as explicative variables all the variables present in the path step  $w_{i-1}$ and, at least, one variable that belongs to the path step $w_i$. \\
Now, upon noting that
\begin{equation}
\bar{R}^2_{w_{i}}=1- MSE_{(w_{i})} \times c
\end{equation}
where $c=\frac{n-1}{\sum_{j=1}^n(Y_j-\bar{y})^2}$
and \textit{MSE} denotes the mean square error computed by considering a set of $o_{r}$ variables 
\footnote{The mean squared error \textit{MSE} is defined as follows
 \begin{equation}
    MSE=\frac{1}{n-o_r-1} \sum_{j=1}^n (Y_j -\hat{f}(X_j))^2
\end{equation}}, 
the inequalities in \eqref{eq:dis} can be reformulated as follows
\begin{equation}\label{eq:mse}
     MSE_{(w_{i-1})} \geq   MSE_{({inter})} \geq   MSE_{(w_{i})} 
\end{equation}
 Eventually, as not only all the variables included in the  best path step $w_i$ may be relevant for the the variable $Y$, the performance of the model can be further improved by discarding from the list of the explicative variables those that are not significant.\\
 \end{lemma}
The above considerations are at the base of the \textit{best path algorithm}, which is described here below. This algorithm belongs to \textit{Sequential Forward Search} because it starts with an empty set and keeps on adding variables \cite{patil2014dimension}. 
The pseudo code of the algorithm is the following:  
 \begin{itemize}
 \item \textbf{Step 0}: Run the algorithm to find the optimal tree or forest using the \cite{edwards2010selecting} algorithm, and call this model $\mathcal{M}_0$
     \item \textbf{Step 1}: Select the variable of interest $Y$ and identify all path steps   $w_i$, starting form the node $Y$, with $i=1,...,k$
     \item \textbf{Step 2}: For  $i=1,...,k$ :
         \subitem (a) fit an econometric model with  the $X_j$ variables as independent variables, where  $X_j \in w_i$   
         \subitem (b) implement \textit{cross-validation}
     \item \textbf{Step 3}: compute  $\bar{R}^2$ for each model and pick as best model the one with the largest $\bar{R}^2$, and call it $ \mathcal{M}_w$
     \item \textbf{Step 4}: Fit the model $\mathcal{M}_w$  and select only the significant variables. Call this model $\mathcal{M}_f$
 \end{itemize}
 It is worth noting that:
 \begin{equation}
 \mathcal{M}_f \subseteq \mathcal{M}_w \subset \mathcal{M}_0. \\
 \end{equation}
 This algorithm is a special case of the \textit{mRMRe } algorithm of \cite{de2013mrmre}, originally proposed by \cite{battiti1994using} that measures the importance of variables based on mutual information criterion to evaluate a set of candidate variables and to select an informative subset \cite{kratzer2018varrank}. \\
 With reference to the algorithm here proposed, it must be noted that at \textbf{Step 0} it operates an unsupervised learning, finalized to highlight the relationships between variables inside the dataset \cite{edwards2010selecting}.\\ At \textbf{Step 1} the algorithm organizes the variables in specific subsets or path steps according to a specific variable of interest.\\  At  \textbf{Step 2} Models with explicative variables included in the path steps devised in the previous step are estimated and a
  cross validation analysis is carried out to avoid problems like overfitting, selection bias. This also guarantee the possibility of   extending model results to other independent  dataset  \cite{cawley2010over}.
 In this stage the ability of the model to predict data non used in the estimation phase, is checked by dividing randomly the set of observations into $h$ groups, or \textit{fold}, of approximately equal size. Each fold, in turn, is treated as validation set \cite{friedman2017elements}.\\
 At \textbf{Step 3}  the algorithm 
identifies the best path step in term of \textit{Adjusted R-squared} $\bar{R}^2$.\\
 Finally, if some variables of the best path are  not significant respect to the variable of interest $Y$, they are removed in the \textbf{Step 4}.\\
 The best path step $w_i$ includes the set of variables which are the most relevant for the explanation of $Y$ not only between all the other path steps but also between all possible subsets of variables included in the dataset.

\section{The Best Path Algorithm at Work}
\label{4}
As a testbed for this theoretical approach, we apply the Best Path Algorithm  by using three different examples of real-world datasets. The first dataset is \textit{Hitters} dataset, which is available on R and holds information from Major League Baseball Data from the 1986 and 1987 seasons \cite{james2013introduction}. This dataset was used to introduce the \textit{LASSO} method as a strategy for variable selections \cite{tibshirani1997lasso}. The second application regards \textit{Communities and Crime} dataset \footnote{https://archive.ics.uci.edu/ml/datasets/communities+and+crime}. It combines socio-economic data 
from the 1990 US Census, law enforcement data from the 1990 US LEMAS survey, and crime data from the 1995 FBI UCR \cite{redmond2002data}. The third dataset arises from the main \textit{World Development Indicators} made available to the World Bank\footnote{https://databank.worldbank.org/source/world-development-indicators}. In order to test the prediction performance of the \textit{best path step algorithm}, we present a prediction comparison with the\textit{LASSO} method which is a benchmark method in statistical learning\cite{tibshirani1997lasso}  .

\subsection{\textit{Hitters}}
The Hitters dataset is composed of $263$ observations and $20$ variables: three of them are dichotomous variables while the others are continuous ones. The aim is to find the best appropriate subset of variables to explain a baseball player’s salary on the basis of various statistics
associated with his performance in the previous year. In order to realize this task, we implement the \textit{best path algorithm}  for the node associated to \textit{Salary}\footnote{1987 annual salary on opening day in thousands of dollars}. 
Figure \ref{fig2} encodes the conditional independence relations among the variables of the dataset. 

\begin{figure}[H]\par\medskip
\centering
\includegraphics[scale=0.60]{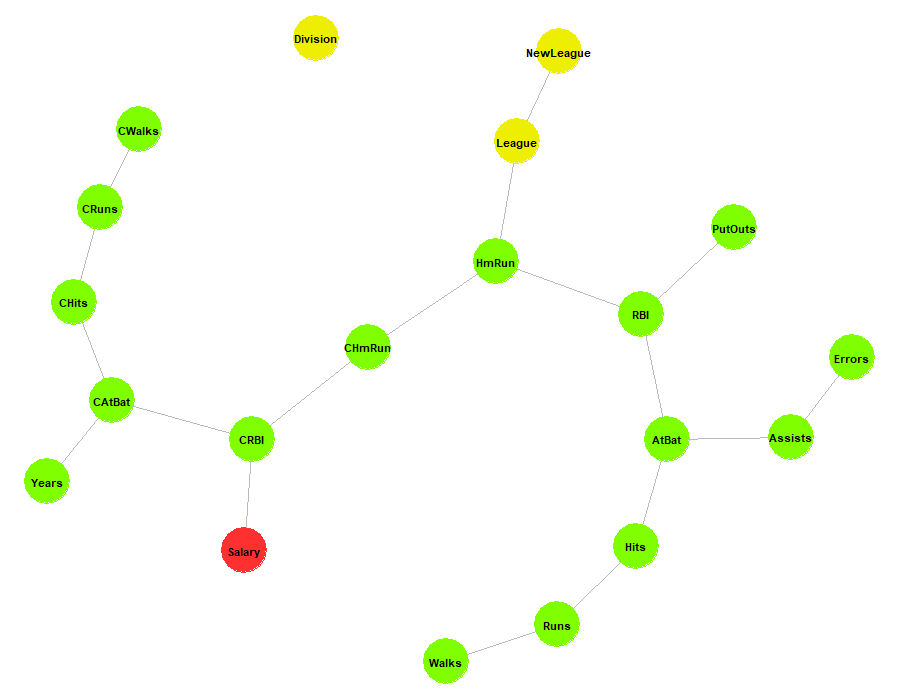}
\caption{The minimal BIC forest for the \textit{Hitters} dataset. Discrete variables are
shown as yellow nodes,
continuous variables are green while the variable of interest \textit{Salary} is the red node. }
\label{fig2}
\end{figure}
As we can see from Figure \ref{fig2}, the \textit{Hitters} optimal forest is depicted through two components, one isolated node (\textit{Division} \footnote{A factor with levels E and W indicating player's division at the end of 1986}) and another big component, which includes the other variables of the dataset. The node of interest, \textit{Salary}, is connected directly with the node \textit{CRIB} representing the number of runs batted in during a players career. Starting from the \textit{Salary} node, the \textit{distance} and the \textit{path steps} are computed, according to the definitions \ref{distance} and \ref{path-s}. The distance between \textit{Salary} and \textit{Walks} \footnote{Number of walks in 1986} turns out to be equal to $d^{(8)}$.  As the later is also the maximum distance 
meaning that $8$ possible subsets of variables (\textit{path steps}) must be considered in order to build an econometric linear model to explain  \textit{Salary}.

\begin{table}[H]\par\medskip
\begin{tabular}{@{}ccc@{}}
\textbf{Path Steps}&\textbf{MSE}&\textbf{$\bar{R}^2$} \\ \midrule
\multicolumn{1}{|c|}{\textit{path step 1}} & \multicolumn{1}{c|}{\textit{367.453}} & \multicolumn{1}{c|}{\textit{0.380}} \\ \midrule
\multicolumn{1}{|c|}{\textit{path step 2}} & \multicolumn{1}{c|}{\textit{374.153}} & \multicolumn{1}{c|}{\textit{0.361}} \\ \midrule
\multicolumn{1}{|c|}{\textit{path step 3}} & \multicolumn{1}{c|}{\textit{360.377}} & \multicolumn{1}{c|}{\textit{0.431}} \\ \midrule
\multicolumn{1}{|c|}{\textit{path step 4}} & \multicolumn{1}{c|}{\textit{357.546}} & \multicolumn{1}{c|}{\textit{0.436}} \\ \midrule
\multicolumn{1}{|c|}{\textit{path step 5}} & \multicolumn{1}{c|}{\textit{353.428}} & \multicolumn{1}{c|}{\textit{0.460}} \\ \midrule
\multicolumn{1}{|c|}{\textit{path step 6}} & \multicolumn{1}{c|}{\textit{350.317}} & \multicolumn{1}{c|}{\textit{0.448}} \\ \midrule
\multicolumn{1}{|c|}{\textit{path step 7}} & \multicolumn{1}{c|}{\textit{350.876}} & \multicolumn{1}{c|}{\textit{0.447}} \\ \midrule
\multicolumn{1}{|c|}{\textit{path step 8}} & \multicolumn{1}{c|}{\textit{344.545}} & \multicolumn{1}{c|}{\textit{0.463}} \\ \bottomrule
\end{tabular}
\caption{$\bar{R}^2$ and \textit{MSE} for each \textit{path steps}; variable of interest: \textit{Salary}}
\label{tab1}
\end{table}

Table \ref{tab1} shows  \textit{MSE} and $\bar{R}^2$, for each \textit{path step}, after implementing cross-validation over \textit{OLS} models. As we can see from results shown in this table, the  algorithm identifies \textit{path step 8} as the \textit{best path step}. The final output of the algorithm is reported in Table \ref{tab2}, where we can see the estimates of the model with explicative variables detected by the algorithm. 

\begin{table}[H]\par\medskip
\begin{tabular}{@{}|cccllc|@{}}
\toprule
\multicolumn{1}{|c|}{\textbf{Coefficients}}      & \multicolumn{1}{c|}{\textbf{Estimate}} & \multicolumn{1}{c|}{\textbf{Std. Error}} & \multicolumn{1}{c|}{\textbf{t value}} & \multicolumn{1}{c|}{\textbf{$Pr(>|t|)$}} & \textbf{Signif.} \\ \midrule
\multicolumn{1}{|c|}{\textit{$\beta_0$}}         & \multicolumn{1}{c|}{\textit{41.83}}    & \multicolumn{1}{c|}{\textit{62.99}}      & \multicolumn{1}{l|}{\textit{0.664}}   & \multicolumn{1}{l|}{\textit{0.507}}      & \textit{}        \\ \midrule
\multicolumn{1}{|c|}{\textit{$\beta_{{CRuns}}$}}   & \multicolumn{1}{c|}{\textit{1.12}}     & \multicolumn{1}{c|}{\textit{0.20}}       & \multicolumn{1}{l|}{\textit{5.535}}   & \multicolumn{1}{l|}{\textit{0.000}}      & \textit{***}     \\ \midrule
\multicolumn{1}{|c|}{\textit{$\beta_{CWalks}$}}  & \multicolumn{1}{c|}{\textit{-0.70}}    & \multicolumn{1}{c|}{\textit{0.27}}       & \multicolumn{1}{l|}{\textit{-2.596}}  & \multicolumn{1}{l|}{\textit{0.010}}      & \textit{**}      \\ \midrule
\multicolumn{1}{|c|}{\textit{$\beta_{AtBat}$}}   & \multicolumn{1}{c|}{\textit{-2.13}}    & \multicolumn{1}{c|}{\textit{0.53}}       & \multicolumn{1}{l|}{\textit{-3.982}}  & \multicolumn{1}{l|}{\textit{0.000}}      & \textit{***}     \\ \midrule
\multicolumn{1}{|c|}{\textit{$\beta_{PutOuts}$}} & \multicolumn{1}{c|}{\textit{0.30}}     & \multicolumn{1}{c|}{\textit{0.08}}       & \multicolumn{1}{l|}{\textit{4.004}}   & \multicolumn{1}{l|}{\textit{0.000}}      & \textit{***}     \\ \midrule
\multicolumn{1}{|c|}{\textit{$\beta_{Hits}$}}    & \multicolumn{1}{c|}{\textit{7.31}}     & \multicolumn{1}{c|}{\textit{1.69}}       & \multicolumn{1}{l|}{\textit{4.335}}   & \multicolumn{1}{l|}{\textit{0.000}}      & \textit{***}     \\ \midrule
\multicolumn{1}{|c|}{\textit{$\beta_{Walks}$}}   & \multicolumn{1}{c|}{\textit{6.17}}     & \multicolumn{1}{c|}{\textit{1.56}}       & \multicolumn{1}{l|}{\textit{3.961}}   & \multicolumn{1}{l|}{\textit{0.000}}      & \textit{***}     \\ \midrule
\multicolumn{6}{|l|}{\textit{\textbf{$\bar{R}^2=0.486$; Signif.   Codes: 0 ‘***’ 0.001 ‘**’ 0.01 ‘*’ 0.05 ‘.’ 0.1}}}                                                                                                                       \\ \bottomrule
\end{tabular}
\caption{Final output; \textit{Salary} model, \textit{OLS} regressions summary }
\label{tab2}
\end{table}

\subsection{\textit{Communities and Crime}}
 The second application employs \textit{Communities and Crime} dataset, that combine data from different sources \cite{redmond2002data}.
 This dataset contains more than $100$ variables, for this reason, the manual variable selection can go beyond the cognitive abilities of researchers and  create fertile ground for scientific malpractices (like reverse p-hacking, \cite{head2015extent, chuard2019evidence}) and for generating ex-post justified opaque models influenced by cognitive biases.
In detail, the original dataset is composed of $128$ variable and $1994$ observations. After removing variables with a lot of missing data, the resulting dataset tursn out to be composed of $103$ variables and $1993$ observation. \cite{redmond2002data} proposed a data driven approach to help police departments when developing a strategic viewpoint for avoiding the crimes. The target variable of this strategy is \textit{Violent crime per population}, that represents the total number of violent crimes per $100K$ population \cite{us20032000}. The aim of this application is to select the appropriate explicative variables in order to build an econometric model for the variable \textit{Violent crime per population}. Figure \ref{fig3} shows the optimal forest built on \textit{Communities and Crime} dataset, the variable label correspondence is shown in Table \ref{NameVar}. The forest in Figure \ref{fig3} is composed of two isolated nodes and one big component. The two isolated nodes are node 1, the yellow one, corresponding to the variable \textit{Community Name}, and the red one, the node 2,  associated with the continuous variable \textit{fold}. This variable  denotes the number for non-random 10 fold cross-validation, potentially useful for debugging, paired tests. The node of interest \textit{Violent crime per population} (103, the red one), belongs to the big component, meaning that only the variables present in this component must be considered.

\begin{figure}[H]\par\medskip
\centering
\includegraphics[scale=0.60]{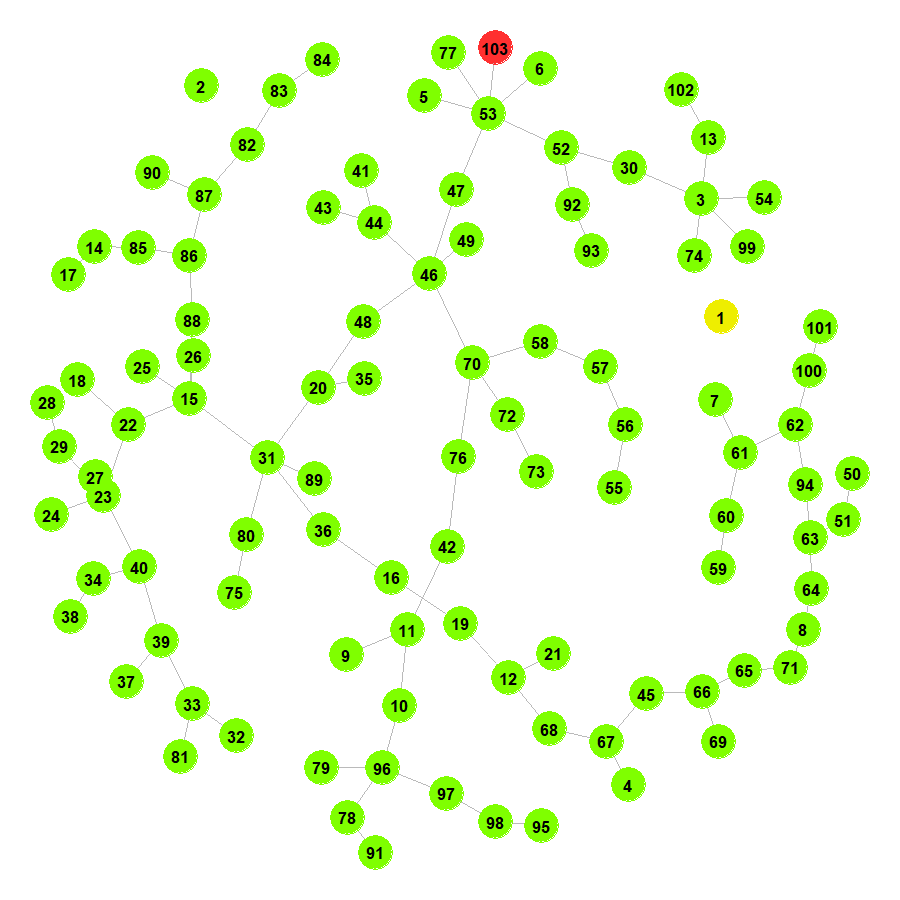}
\caption{The minimal BIC forest for the \textit{Communities and Crime} dataset. Discrete variable is
shown as yellow nodes,
continuous variables are green while the variable of interest \textit{Violent crime per Population} is represented by the red node. }
\label{fig3}
\end{figure}
According to definition \ref{path-s}, we contemplate $22$ potential  subsets of variables (path steps) that can explain the variable \textit{Violent crime per population}. Table \ref{tab3} shows  \textit{MSE} and $\bar{R}^2$, for each model including as explicative variables the ones corresponding to each \textit{path step}, after   cross-validation has been carried foe each of them. As we can see from Table \ref{tab3}, the  algorithm identifies \textit{path step 20} as the \textit{best path step}. The final output of the algorithm is reported in Table \ref{tab4}. As we can see the algorithm suggests the use of 23 variables to explain \textit{Violent crime per population}, where some of them have a negative effect such as \textit{PctWorkMom} which is the percentage of moms of kids under 18 age, while others have a positive effect such as the \textit{NumStreet} which is the number of homeless people counted in the street.

\begin{table}[H]\par\medskip
\begin{tabular}{@{}|c|l|l|@{}}
\textbf{Path Steps}   & \multicolumn{1}{c|}{\textbf{MSE}} & \multicolumn{1}{c|}{\textbf{$\bar{R}^2$}} \\ \midrule
\textit{path step 1}  & \textit{0.1569}                   & \textit{0.5474}                           \\ \midrule
\textit{path step 2}  & \textit{0.1438}                   & \textit{0.6204}                           \\ \midrule
\textit{path step 3}  & \textit{0.1432}                   & \textit{0.6234}                           \\ \midrule
\textit{path step 4}  & \textit{0.1413}                   & \textit{0.6336}                           \\ \midrule
\textit{path step 5}  & \textit{0.1391}                   & \textit{0.6447}                           \\ \midrule
\textit{path step 6}  & \textit{0.1392}                   & \textit{0.6441}                           \\ \midrule
\textit{path step 7}  & \textit{0.1394}                   & \textit{0.6429}                           \\ \midrule
\textit{path step 8}  & \textit{0.1394}                   & \textit{0.6429}                           \\ \midrule
\textit{path step 9}  & \textit{0.1389}                   & \textit{0.6453}                           \\ \midrule
\textit{path step 10} & \textit{0.1379}                   & \textit{0.6504}                           \\ \midrule
\textit{path step 11} & \textit{0.1367}                   & \textit{0.6564}                           \\ \midrule
\textit{path step 12} & \textit{0.1367}                   & \textit{0.6559}                           \\ \midrule
\textit{path step 13} & \textit{0.1370}                   & \textit{0.6544}                           \\ \midrule
\textit{path step 14} & \textit{0.1370}                   & \textit{0.6542}                           \\ \midrule
\textit{path step 15} & \textit{0.1368}                   & \textit{0.6554}                           \\ \midrule
\textit{path step 16} & \textit{0.1364}                   & \textit{0.6572}                           \\ \midrule
\textit{path step 17} & \textit{0.1366}                   & \textit{0.6567}                           \\ \midrule
\textit{path step 18} & \textit{0.1365}                   & \textit{0.6569}                           \\ \midrule
\textit{path step 19} & \textit{0.1365}                   & \textit{0.6571}                           \\ \midrule
\textit{path step 20} & \textit{0.1359}                   & \textit{0.6601}                           \\ \midrule
\textit{path step 21} & \textit{0.1359}                   & \textit{0.6599}                           \\ \midrule
\textit{path step 22} & \textit{0.1360}                   & \textit{0.6597}                           \\ \midrule
\textit{path step 23} & \textit{0.1359}                   & \textit{0.6599}                           \\ \midrule
\textit{path step 24} & \textit{0.1360}                   & \textit{0.6593}                           \\ \bottomrule
\end{tabular}
\caption{$\bar{R}^2$ and \textit{MSE} for each \textit{path steps}; variable of interest: \textit{Violent crime per population}}
\label{tab3}
\end{table}

\begin{table}[H]\par\medskip
\begin{tabular}{@{}|cccccc|@{}}
\toprule
\multicolumn{1}{|c|}{\textbf{Coefficients}}                    & \multicolumn{1}{c|}{\textbf{Estimate}} & \multicolumn{1}{c|}{\textbf{Std. Error}} & \multicolumn{1}{c|}{\textbf{t value}} & \multicolumn{1}{c|}{\textbf{$Pr(>|t|)$}} & \textbf{Signif} \\ \midrule
\multicolumn{1}{|c|}{\textit{$\beta_0$}}                       & \multicolumn{1}{c|}{\textit{0.459}}    & \multicolumn{1}{c|}{\textit{0.075}}      & \multicolumn{1}{c|}{\textit{6.083}}   & \multicolumn{1}{c|}{\textit{0.000}}      & \textit{***}    \\ \midrule
\multicolumn{1}{|c|}{\textit{$\beta_{PctIlleg}$}}              & \multicolumn{1}{c|}{\textit{0.173}}    & \multicolumn{1}{c|}{\textit{0.038}}      & \multicolumn{1}{c|}{\textit{4.499}}   & \multicolumn{1}{c|}{\textit{0.000}}      & \textit{***}    \\ \midrule
\multicolumn{1}{|c|}{\textit{$\beta_{racepctblack}$}}          & \multicolumn{1}{c|}{\textit{0.196}}    & \multicolumn{1}{c|}{\textit{0.026}}      & \multicolumn{1}{c|}{\textit{7.403}}   & \multicolumn{1}{c|}{\textit{0.000}}      & \textit{***}    \\ \midrule
\multicolumn{1}{|c|}{\textit{$\beta_{PctKids2Par}$}}           & \multicolumn{1}{c|}{\textit{-0.335}}   & \multicolumn{1}{c|}{\textit{0.063}}      & \multicolumn{1}{c|}{\textit{-5.354}}  & \multicolumn{1}{c|}{\textit{0.000}}      & \textit{***}    \\ \midrule
\multicolumn{1}{|c|}{\textit{$\beta_{PctVacantBoarded}$}}      & \multicolumn{1}{c|}{\textit{0.039}}    & \multicolumn{1}{c|}{\textit{0.018}}      & \multicolumn{1}{c|}{\textit{2.095}}   & \multicolumn{1}{c|}{\textit{0.036}}      & \textit{*}      \\ \midrule
\multicolumn{1}{|c|}{\textit{$\beta_{NumStreet}$}}             & \multicolumn{1}{c|}{\textit{0.210}}    & \multicolumn{1}{c|}{\textit{0.043}}      & \multicolumn{1}{c|}{\textit{4.882}}   & \multicolumn{1}{c|}{\textit{0.000}}      & \textit{***}    \\ \midrule
\multicolumn{1}{|c|}{\textit{$\beta_{MalePctDivorce}$}}        & \multicolumn{1}{c|}{\textit{0.113}}    & \multicolumn{1}{c|}{\textit{0.032}}      & \multicolumn{1}{c|}{\textit{3.498}}   & \multicolumn{1}{c|}{\textit{0.000}}      & \textit{***}    \\ \midrule
\multicolumn{1}{|c|}{\textit{$\beta_{NumImmig}$}}              & \multicolumn{1}{c|}{\textit{-0.181}}   & \multicolumn{1}{c|}{\textit{0.058}}      & \multicolumn{1}{c|}{\textit{-3.118}}  & \multicolumn{1}{c|}{\textit{0.002}}      & \textit{**}     \\ \midrule
\multicolumn{1}{|c|}{\textit{$\beta_{HousVacant}$}}            & \multicolumn{1}{c|}{\textit{0.178}}    & \multicolumn{1}{c|}{\textit{0.030}}      & \multicolumn{1}{c|}{\textit{6.043}}   & \multicolumn{1}{c|}{\textit{0.000}}      & \textit{***}    \\ \midrule
\multicolumn{1}{|c|}{\textit{$\beta_{PctPopUnderPov}$}}        & \multicolumn{1}{c|}{\textit{-0.124}}   & \multicolumn{1}{c|}{\textit{0.038}}      & \multicolumn{1}{c|}{\textit{-3.235}}  & \multicolumn{1}{c|}{\textit{0.001}}      & \textit{**}     \\ \midrule
\multicolumn{1}{|c|}{\textit{$\beta_{PctEmploy}$}}             & \multicolumn{1}{c|}{\textit{0.103}}    & \multicolumn{1}{c|}{\textit{0.053}}      & \multicolumn{1}{c|}{\textit{1.960}}   & \multicolumn{1}{c|}{\textit{0.050}}      & \textit{.}      \\ \midrule
\multicolumn{1}{|c|}{\textit{$\beta_{MedRent}$}}               & \multicolumn{1}{c|}{\textit{0.285}}    & \multicolumn{1}{c|}{\textit{0.054}}      & \multicolumn{1}{c|}{\textit{5.279}}   & \multicolumn{1}{c|}{\textit{0.000}}      & \textit{***}    \\ \midrule
\multicolumn{1}{|c|}{\textit{$\beta_{pctWWage}$}}              & \multicolumn{1}{c|}{\textit{-0.213}}   & \multicolumn{1}{c|}{\textit{0.044}}      & \multicolumn{1}{c|}{\textit{-4.790}}  & \multicolumn{1}{c|}{\textit{0.000}}      & \textit{***}    \\ \midrule
\multicolumn{1}{|c|}{\textit{$\beta_{whitePerCap}$}}           & \multicolumn{1}{c|}{\textit{-0.103}}   & \multicolumn{1}{c|}{\textit{0.032}}      & \multicolumn{1}{c|}{\textit{-3.226}}  & \multicolumn{1}{c|}{\textit{0.001}}      & \textit{**}     \\ \midrule
\multicolumn{1}{|c|}{\textit{$\beta_{RentLowQ}$}}              & \multicolumn{1}{c|}{\textit{-0.251}}   & \multicolumn{1}{c|}{\textit{0.051}}      & \multicolumn{1}{c|}{\textit{-4.919}}  & \multicolumn{1}{c|}{\textit{0.000}}      & \textit{***}    \\ \midrule
\multicolumn{1}{|c|}{\textit{$\beta_{OtherPerCap}$}}           & \multicolumn{1}{c|}{\textit{0.000}}    & \multicolumn{1}{c|}{\textit{0.000}}      & \multicolumn{1}{c|}{\textit{2.495}}   & \multicolumn{1}{c|}{\textit{0.013}}      & \textit{*}      \\ \midrule
\multicolumn{1}{|c|}{\textit{$\beta_{pctUrban}$}}              & \multicolumn{1}{c|}{\textit{0.038}}    & \multicolumn{1}{c|}{\textit{0.009}}      & \multicolumn{1}{c|}{\textit{4.387}}   & \multicolumn{1}{c|}{\textit{0.000}}      & \textit{***}    \\ \midrule
\multicolumn{1}{|c|}{\textit{$\beta_{pctWRetire}$}}            & \multicolumn{1}{c|}{\textit{-0.100}}   & \multicolumn{1}{c|}{\textit{0.030}}      & \multicolumn{1}{c|}{\textit{-3.351}}  & \multicolumn{1}{c|}{\textit{0.001}}      & \textit{***}    \\ \midrule
\multicolumn{1}{|c|}{\textit{$\beta_{MedOwnCostPctIncNoMtg}$}} & \multicolumn{1}{c|}{\textit{-0.084}}   & \multicolumn{1}{c|}{\textit{0.019}}      & \multicolumn{1}{c|}{\textit{-4.478}}  & \multicolumn{1}{c|}{\textit{0.000}}      & \textit{***}    \\ \midrule
\multicolumn{1}{|c|}{\textit{$\beta_{pctWFarmSelf}$}}          & \multicolumn{1}{c|}{\textit{0.034}}    & \multicolumn{1}{c|}{\textit{0.019}}      & \multicolumn{1}{c|}{\textit{1.810}}   & \multicolumn{1}{c|}{\textit{0.070}}      & \textit{.}      \\ \midrule
\multicolumn{1}{|c|}{\textit{$\beta_{PctPersDenseHous}$}}      & \multicolumn{1}{c|}{\textit{0.227}}    & \multicolumn{1}{c|}{\textit{0.042}}      & \multicolumn{1}{c|}{\textit{5.389}}   & \multicolumn{1}{c|}{\textit{0.000}}      & \textit{***}    \\ \midrule
\multicolumn{1}{|c|}{\textit{$\beta_{PctNotSpeakEnglWell}$}}   & \multicolumn{1}{c|}{\textit{-0.116}}   & \multicolumn{1}{c|}{\textit{0.047}}      & \multicolumn{1}{c|}{\textit{-2.444}}  & \multicolumn{1}{c|}{\textit{0.015}}      & \textit{*}      \\ \midrule
\multicolumn{1}{|c|}{\textit{$\beta_{PctWorkMom}$}}            & \multicolumn{1}{c|}{\textit{-0.119}}   & \multicolumn{1}{c|}{\textit{0.024}}      & \multicolumn{1}{c|}{\textit{-5.028}}  & \multicolumn{1}{c|}{\textit{0.000}}      & \textit{***}    \\ \midrule
\multicolumn{1}{|c|}{\textit{$\beta_{PctForeignBorn}$}}        & \multicolumn{1}{c|}{\textit{0.123}}    & \multicolumn{1}{c|}{\textit{0.038}}      & \multicolumn{1}{c|}{\textit{3.285}}   & \multicolumn{1}{c|}{\textit{0.001}}      & \textit{**}     \\ \midrule
\multicolumn{6}{|c|}{\textit{$\bar{R}^2=673$;   Signif.   Codes: 0 ‘***’ 0.001 ‘**’   0.01 ‘*’ 0.05 ‘.’ 0.1}}                                                                                                                                           \\ \bottomrule
\end{tabular}
\caption{Final output; \textit{Violent crime per population} model, \textit{OLS} regressions summary }
\label{tab4}
\end{table}

\subsection{\textit{World Development Indicators}}
The third application is based on World Development Indicators available on the world bank database \cite{world2014world}. The dataset has been obtained  downloading all development indicators available for $2020$, and it is composed of $266$ observations, Nations and Community Organizations(ex. European Union, OECD members, and so on...), and $1444$ indicators. After removing  indicators and observations with a a lot of missing, the final dataset turns out to be composed of $166$ observations, only Nations, and $133$ indicators. Figure \ref{fig4} shows the Nations within the dataset, while the list of the indicators are reported in Table \ref{NameVar1}.

\begin{figure}[H]\par\medskip
\centering
\includegraphics[scale=0.45]{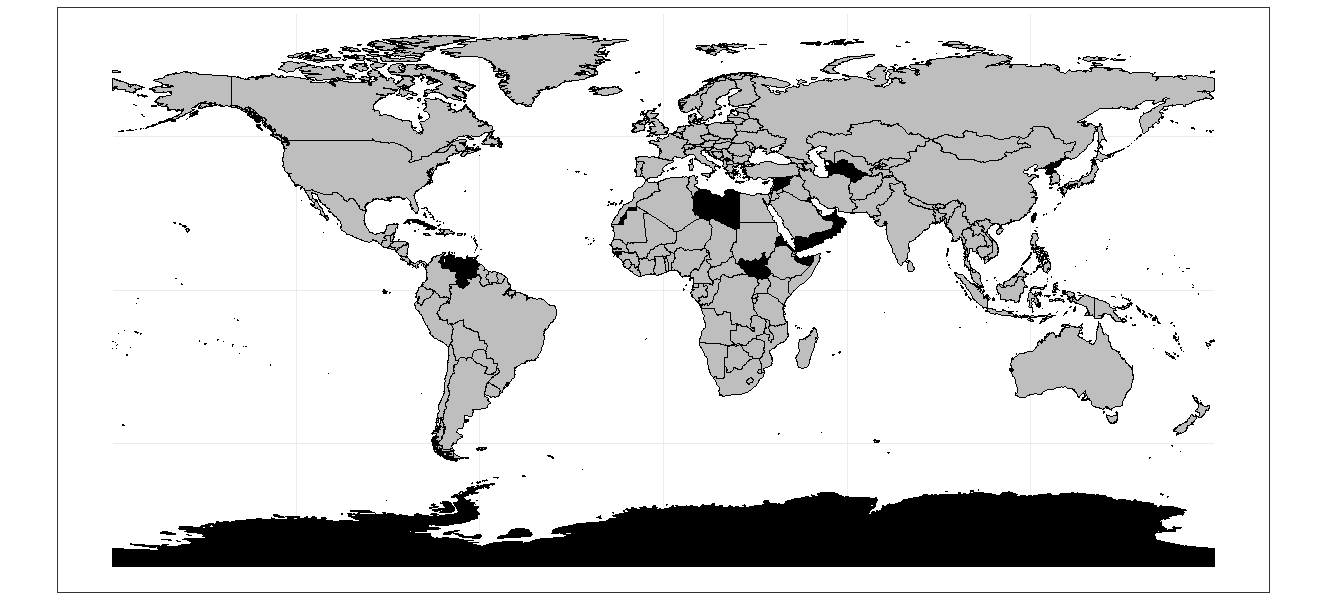}
\caption{Map of the nations within the \textit{World Development Indicators} dataset, in grey are indicated Nations included in the dataset, in black the absent ones.}
\label{fig4}
\end{figure}
 
The aim of this application is to select the appropriate explicative variables to build an econometric model for \textit{Price level ratio}  \cite{clague1986determinants, chow1987money} which is the ratio of a purchasing power parity (PPP) conversion factor to an exchange rate. It provides a measure of the differences in price levels between countries by indicating the number of units of the common currency needed to buy the same volume of a given good in other countries. The impressive body of data on PPP that has recently become available has stimulated renewed interest in the explanation of the discrepancies and has afforded the opportunity to test some hypotheses on this variable \cite{clague1986determinants}. For instance, the idea that in long-run equilibrium PPPs tend to equal exchange rates, as suggested by \cite{clague1986determinants}, has been rather conclusively refuted by the data. Indeed, data reveal a very systematic tendency for poor countries to have lower PPP  than exchange rate. Different theoretical works on national price levels \cite{kravis1983toward,clague1985model} have suggested other variables to explain PPP, such as country characteristic and both long-run and short-run variables. Structural variables include the foreign trade ratio, the service share in GDP, natural resource abundance, tourist receipts, and educational attainment, while short-run or monetary variables include the growth in the money supply and the balance of trade. The graph in Figure \ref{fig5} encodes the dependence among the variables present in the \textit{World Development Indicators} dataset through a tree. The node of interest PPP, is directly connected with the node $29$, that corresponds to \textit{GNI per capita, Atlas method} (formerly GNP per capita). This variable is the gross national income, converted to U.S. dollars using the World Bank Atlas method. GNI is the sum of value added by all resident producers plus any product taxes (less subsidies), not included in the valuation of output, plus net receipts of primary income (compensation of employees and property income) from abroad. GNI, calculated in national currency, is usually converted to U.S. dollars at official exchange rates for comparisons across economies, although an alternative rate is used when the official exchange rate is judged to diverge by an exceptionally large margin from the rate actually applied in international transactions. To smooth fluctuations in prices and exchange rates, a special \textit{Atlas} method  of conversion is used by the World Bank. This applies a conversion factor that averages the exchange rate for a given year and the two preceding years, adjusted for differences in rates of inflation between the country and  a set of countries which, up to 2000, were the G-5 countries (France, Germany, Japan, the United Kingdom, and the United States) up to 2000. Since 2001, the set of countries include Euro area, Japan, United Kingdom, and the United States. \footnote{https://databank.worldbank.org/source/world-development-indicators}

\begin{figure}[H]\par\medskip
\centering
\includegraphics[scale=0.60]{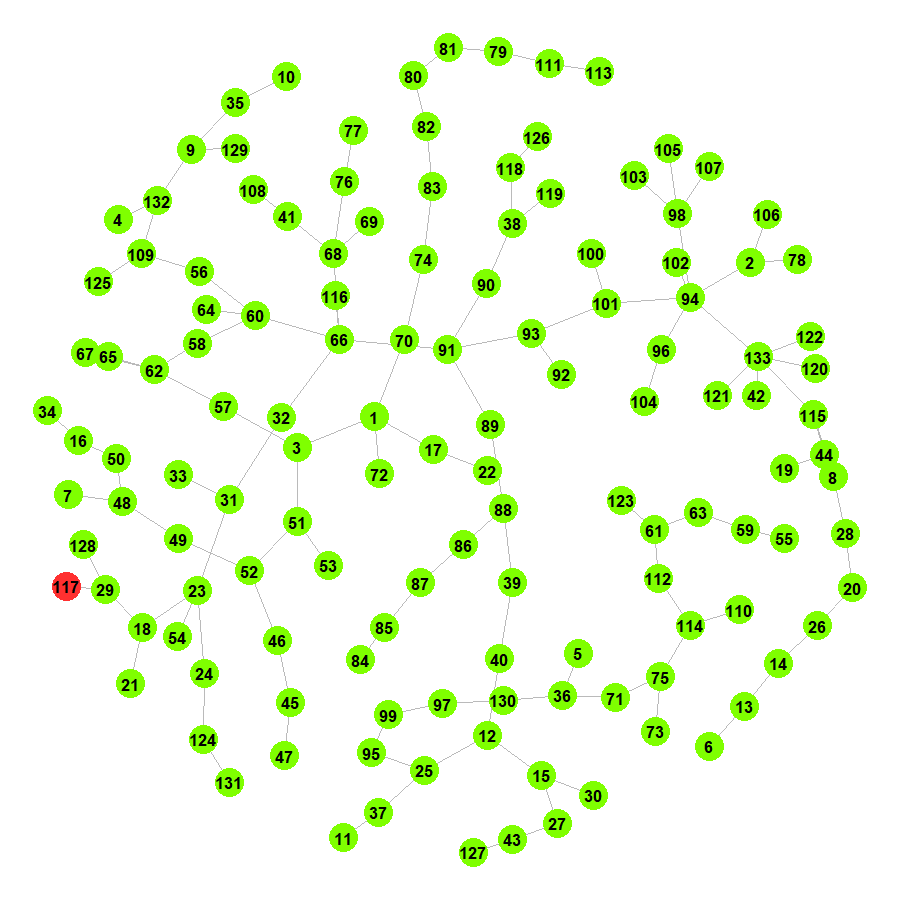}
\caption{The minimal BIC tree for the \textit{World Development Indicators} dataset. Continuous variables are in green while the variable of interest \textit{Price level ratio} is  the red node. }
\label{fig5}
\end{figure}
Table \ref{tab5} shows \textit{MSE} and $\bar{R}^2$, for each model including as explixcative variables the ones present in the path steps, after implementing cross-validation on the estimated models by \textit{OLS}. The algorithm identifies \textit{path step 4} as the the \textit{best path step} for the variable of interest \textit{Price level ratio}. The final output of the algorithm is reported in Table \ref{tab6}. As we see from this table, the algorithm suggests the use of $6$ variables to explain the variable \textit{Price level ratio}. The variables \textit{GNI per capita, Atlas method}, \textit{GDP per capita curent }\footnote{GDP per capita is gross domestic product divided by midyear population. GDP is the sum of gross value added by all resident producers in the economy plus any product taxes and minus any subsidies not included in the value of the products. It is calculated without making deductions for depreciation of fabricated assets or for depletion and degradation of natural resources. Data are in current U.S. dollars.} and \textit{Human capital index}\footnote{The HCI calculates the contributions of health and education to worker productivity. The final index score ranges from zero to one and measures the productivity as a future worker of child born today relative to the benchmark of full health and complete education.} increase the \textit{Price level ratio}. \\While, the variables \textit{GDP per capita 2015}
 \textit{GDP per capita, PPP, 2017}
  \footnote{GDP per capita based on purchasing power parity (PPP) is the gross domestic product converted to international dollars using purchasing power parity rates. An international dollar has the same purchasing power over GDP as the U.S. dollar has in the United States. GDP at purchaser's prices is the sum of gross value added by all resident producers in the country plus any product taxes and minus any subsidies not included in the value of the products. It is calculated without making deductions for depreciation of fabricated assets or for depletion and degradation of natural resources. Data are in constant 2017 international dollars.}
  and \textit{Personal remittances, received}\footnote{Personal remittances comprise personal transfers and compensation of employees. Personal transfers consist of all current transfers in cash or in kind made or received by resident households to or from nonresident households. Personal transfers thus include all current transfers between resident and nonresident individuals. Compensation of employees refers to the income of border, seasonal, and other short-term workers who are employed in an economy where they are not resident and of residents employed by nonresident entities. Data are the sum of two items defined in the sixth edition of the IMF's Balance of Payments Manual: personal transfers and compensation of employees.} have a negative effect for the variable of interest \textit{Price level ratio}.

\begin{table}[H]\par\medskip
\begin{tabular}{@{}|l|l|l|@{}}

\toprule
\multicolumn{1}{|c|}{\textbf{Path Steps}} & \multicolumn{1}{c|}{\textbf{MSE}} & \multicolumn{1}{c|}{\textbf{$\bar{R}^2$}} \\ \midrule
\textit{path step 1}                      & \textit{0.114}                    & \textit{0.716}                            \\ \midrule
\textit{path step 2}                      & \textit{0.115}                    & \textit{0.711}                            \\ \midrule
\textit{path step 3}                      & \textit{0.106}                    & \textit{0.771}                            \\ \midrule
\textit{path step 4}                      & \textit{0.103}                    & \textit{0.774}                            \\ \midrule
\textit{path step 5}                      & \textit{0.108}                    & \textit{0.744}                            \\ \midrule
\textit{path step 6}                      & \textit{0.108}                    & \textit{0.742}                            \\ \midrule
\textit{path step 7}                      & \textit{0.108}                    & \textit{0.735}                            \\ \midrule
\textit{path step 8}                      & \textit{0.119}                    & \textit{0.703}                            \\ \midrule
\textit{path step 9}                      & \textit{0.12}                     & \textit{0.687}                            \\ \midrule
\textit{path step 10}                     & \textit{0.109}                    & \textit{0.73}                             \\ \midrule
\textit{path step 11}                     & \textit{0.124}                    & \textit{0.665}                            \\ \midrule
\textit{path step 12}                     & \textit{0.131}                    & \textit{0.658}                            \\ \midrule
\textit{path step 13}                     & \textit{0.192}                    & \textit{0.607}                            \\ \midrule
\textit{path step 14}                     & \textit{0.201}                    & \textit{0.539}                            \\ \midrule
\textit{path step 15}                     & \textit{0.254}                    & \textit{0.536}                            \\ \midrule
\textit{path step 16}                     & \textit{0.288}                    & \textit{0.503}                            \\ \midrule
\textit{path step 17}                     & \textit{0.236}                    & \textit{0.557}                            \\ \midrule
\textit{path step 18}                     & \textit{0.281}                    & \textit{0.493}                            \\ \midrule
\textit{path step 19}                     & \textit{0.416}                    & \textit{0.447}                            \\ \midrule
\textit{path step 20}                     & \textit{0.416}                    & \textit{0.412}                            \\ \midrule
\textit{path step 21}                     & \textit{0.552}                    & \textit{0.39}                             \\ \midrule
\textit{path step 22}                     & \textit{0.756}                    & \textit{0.355}                            \\ \midrule
\textit{path step 23}                     & \textit{0.756}                    & \textit{0.355}                            \\ \midrule
\textit{path step 24}                     & \textit{0.756}                    & \textit{0.355}                            \\ \midrule
\textit{path step 25}                     & \textit{0.756}                    & \textit{0.355}                            \\ \midrule
\textit{path step 26}                     & \textit{0.784}                    & \textit{0.338}                            \\ \bottomrule
\end{tabular}
\caption{$\bar{R}^2$ and \textit{MSE} for each \textit{path steps}; variable of interest: \textit{Price level ratio}}
\label{tab5}
\end{table}

\begin{table}[H]\par\medskip
\begin{tabular}{@{}|ccccccc|@{}}
\toprule
\multicolumn{1}{|c|}{\textbf{Coefficients}}                                & \multicolumn{1}{c|}{\textbf{Label node}} & \multicolumn{1}{c|}{\textbf{Estimate}}  & \multicolumn{1}{c|}{\textbf{Std. Error}} & \multicolumn{1}{c|}{\textbf{t value}} & \multicolumn{1}{c|}{\textbf{$Pr(>|t|)$}} & \textbf{Signif} \\ \midrule
\multicolumn{1}{|c|}{\textit{$\beta_0$}}                                   & \multicolumn{1}{c|}{\textit{-}}          & \multicolumn{1}{c|}{\textit{3.11E-01}}  & \multicolumn{1}{c|}{\textit{4.52E-02}}   & \multicolumn{1}{c|}{\textit{6.89}}    & \multicolumn{1}{c|}{\textit{1.30E-10}}   & \textit{***}    \\ \midrule
\multicolumn{1}{|c|}{\textit{$\beta_{GNI-Per-Capita}$}}                   & \multicolumn{1}{c|}{\textit{29}}         & \multicolumn{1}{c|}{\textit{1.22E-05}}  & \multicolumn{1}{c|}{\textit{2.72E-06}}   & \multicolumn{1}{c|}{\textit{4.5}}     & \multicolumn{1}{c|}{\textit{1.30E-05}}   & \textit{***}    \\ \midrule
\multicolumn{1}{|c|}{\textit{$\beta_{GDP-Per-Capita-2015}$}}               & \multicolumn{1}{c|}{\textit{18}}         & \multicolumn{1}{c|}{\textit{-8.20E-06}} & \multicolumn{1}{c|}{\textit{4.72E-06}}   & \multicolumn{1}{c|}{\textit{-1.74}}   & \multicolumn{1}{c|}{\textit{0.084}}      & \textit{.}      \\ \midrule
\multicolumn{1}{|c|}{\textit{$\beta_{GDP-Per-Capita-current}$}}            & \multicolumn{1}{c|}{\textit{21}}         & \multicolumn{1}{c|}{\textit{1.37E-05}}  & \multicolumn{1}{c|}{\textit{3.79E-06}}   & \multicolumn{1}{c|}{\textit{3.62}}    & \multicolumn{1}{c|}{\textit{0.0004}}     & \textit{***}    \\ \midrule
\multicolumn{1}{|c|}{\textit{$\beta_{GDP-PPP-2017}$}}                      & \multicolumn{1}{c|}{\textit{23}}         & \multicolumn{1}{c|}{\textit{-1.03E-05}} & \multicolumn{1}{c|}{\textit{1.43E-06}}   & \multicolumn{1}{c|}{\textit{-7.19}}   & \multicolumn{1}{c|}{\textit{2.40E-11}}   & \textit{***}    \\ \midrule
\multicolumn{1}{|c|}{\textit{$\beta_{Human-capital-index}$}}                & \multicolumn{1}{c|}{\textit{31}}         & \multicolumn{1}{c|}{\textit{2.87E-01}}  & \multicolumn{1}{c|}{\textit{1.00E-01}}   & \multicolumn{1}{c|}{\textit{2.86}}    & \multicolumn{1}{c|}{\textit{0.0048}}     & \textit{**}     \\ \midrule
\multicolumn{1}{|c|}{\textit{$\beta_{Personal-remittances-received}$}} & \multicolumn{1}{c|}{\textit{54}}         & \multicolumn{1}{c|}{\textit{-2.58E-03}} & \multicolumn{1}{c|}{\textit{1.23E-03}}   & \multicolumn{1}{c|}{\textit{-2.1}}    & \multicolumn{1}{c|}{\textit{0.0374}}     & \textit{*}      \\ \midrule
\multicolumn{7}{|c|}{\textit{$\bar{R}^2= 0.792$;   Signif.   Codes: 0 ‘***’ 0.001 ‘**’   0.01 ‘*’ 0.05 ‘.’ 0.1}}                                           
\end{tabular}
\caption{Final output;\textit{Price level ratio} model, \textit{OLS} regressions summary }
\label{tab6}
\end{table}

\subsection{Prediction Test}

In statistical learning different methods have been implemented to select automatically variables. The basic strategies used for selecting variables can be divided into two main veins:  \textit{subset selection methods} and \textit{shrinkage methods} \cite{patil2014dimension}.  \textit{Subset selection} algorithms select a subset of predictors for a specific target variable. Different measures and different algorithms implement this strategy. Among them, there are \textit{subset selection} the \textit{Best Subset Selection } \cite{bertsimas2016best}, \textit{Branch and Bound}\cite{narendra1977branch} and \textit{Sequential Methods} \cite{haq2019heart}. Though  \textit{ subset selection} algorithms are  simple and conceptually appealing,  they suffer from computational limitations. In fact, the number of possible
models that must be considered $(2^p )$ grow rapidly as \textit{p} (variables) increases \cite{james2013introduction}. For this reason, algorithms that belong to \textit{shrinkage methods} represent a functional alternative. Indeed these algorithms contain all $p$ predictors and use a technique that constrains or regularizes the coefficient estimates, or equivalently, that shrinks the coefficient estimates towards zero. Among the algorithms that belong to the \textit{shrinkage method}, the \textit{LASSO} method \cite{tibshirani1997lasso} is probably  the most known \cite[among others]{hans2009bayesian,meier2008group,tibshirani2011regression}. In general, the least absolute shrinkage and selection operator (\textit{LASSO}) was put forwarded by \cite{tibshirani1997lasso}, for parameter estimation and also variable (model) selection simultaneously in regression analysis. The \textit{LASSO} is a particular case of the penalized least squares regression with $L1$-penalty function \cite{muthukrishnan2016lasso}. 
The \textit{LASSO} estimate can be defined by:
\begin{equation}
\centering
    \hat{\beta}^{LASSO}= \arg \min _{\beta} \left\{ 
    \frac{1}{2} \sum_{i=1}^n \left(y_i-\beta_0-\sum_{j=1}^p x_{ij} \beta_j
    \right)^2+ \lambda\sum_{j=1}^p |\beta_j|
    \right\}
\end{equation}
where $\lambda \leq 0$, is a turning parameter, to be determined separately\footnote{A method to determine the value of \begin{math}
\lambda
\end{math}, is to choose a gird of \begin{math}
\lambda
\end{math} values, and compute the cross-validation error for each value of \begin{math}
\lambda
\end{math}, and chose the \begin{math}
\lambda
\end{math} associate with the smaller cross validation error.}. \textit{LASSO} transforms each and every coefficient by a
constant turning parameter $\lambda$, truncating at zero. Hence it is a
forward-looking variable selection method for regression. It
decreases the residual sum of squares subject to the sum of the
absolute value of the coefficients being less than a constant \cite{muthukrishnan2016lasso}.
In order to test the accuracy of the \textit{ best path step algorithm}, we do a comparison in terms of prediction (low \textit{MSE}) with the \textit{LASSO} method, for each model presented in this work. In detail,  after implementing the \textit{LASSO} method and the \textit{ best path step algorithm}, for each variable of interest, we split the datasets in $70\%$ train and $30\%$ test in order to test the accuracy performance of these methods.
We repeat this procedure 100 times with the purpose of avoiding problems like overfitting or selection bias. Figure \ref{Prediction} shows the result of prediction, in \textit{Salary} model \ref{main:a} the \textit{best path algorithm} has a lowest \textit{MSE} respect to 
\textit{LASSO} method $68$ times. The model \textit{Crime} is reported in figure \ref{main:b}, in this case the \textit{MSE} of the \textit{best path algorithm} is $100$ times lowest than \textit{LASSO} method. Eventually, figure \ref{main:c} shows that \textit{MSE} is $81$ times lowest with the \textit{best path algorithm} respect to the \textit{LASSO} method in the \textit{Price level ratio} model.

\begin{figure}[H]\par\medskip
\centering
\begin{minipage}{.5\linewidth}
\centering
\subfloat[]{\label{main:a}\includegraphics[scale=.30]{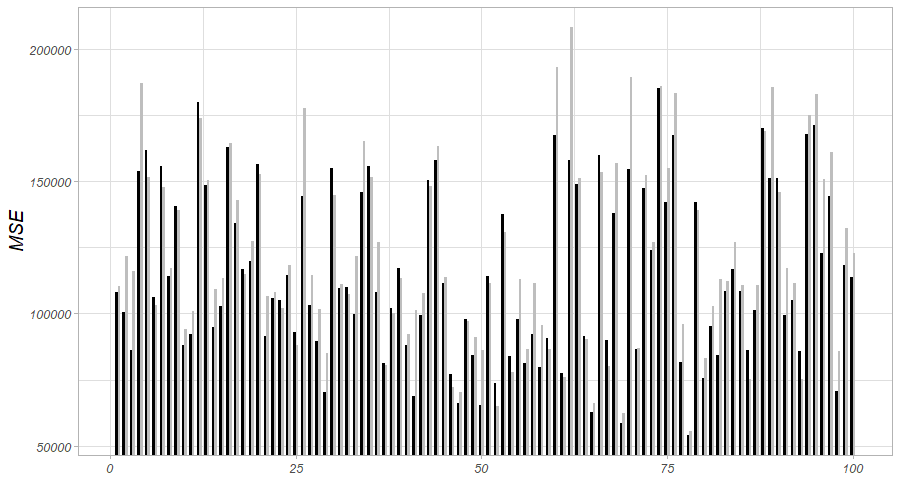}}
\end{minipage}%
\begin{minipage}{.5\linewidth}
\centering
\subfloat[]{\label{main:b}\includegraphics[scale=.30
]{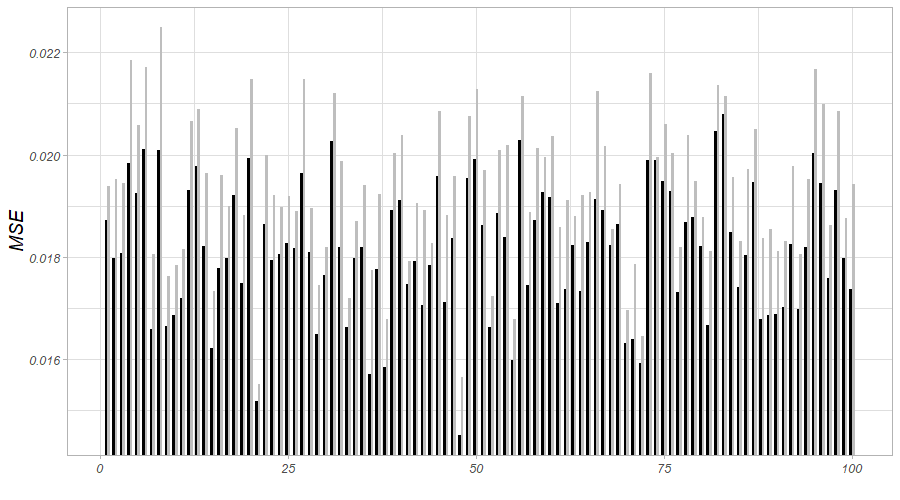}}
\end{minipage}\par\medskip
\centering
\subfloat[]{\label{main:c}\includegraphics[scale=.40]{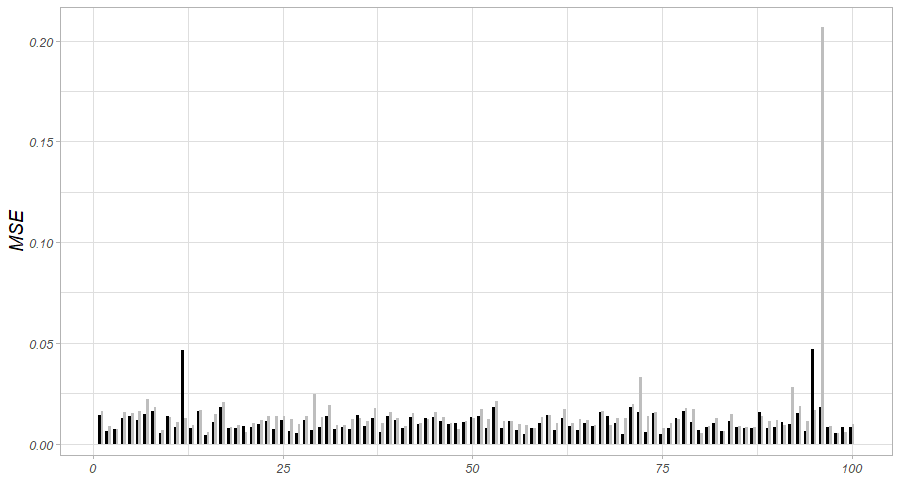}}
\caption{Comparison of the prediction performance between the \textit{best path algorithm} (black lines) and the \textit{LASSO} method (grey lines).
(a) \textit{Salary} model, (b) \textit{Violent Crime per population} model, (c)  \textit{Price level ratio} model }
\label{Prediction}
\end{figure}

\section{Conclusion}
\label{5}
In the last decade, the availability of the large dataset has made difficult the process of variable selection. Different strategies have been developed to mitigate the 'embarrassment of riches' in the choice of the variables \cite{eklund2007embarrassment, altman2018curse}. Most of them arise from machine learning or from statistical learning. And despite there has been an increasing trend in healthcare \cite{wiens2018machine},
economic \cite[among others]{athey2019machine,roth2004generalized}
and biological \cite{camacho2018next}
to leverage machine learning or statistical learning  for high-stakes prediction applications that deeply impact human lives, many of the machine learning and statistical learning models are black boxes that do not explain their predictions in a way that humans can understand the results \cite{rudin2019stop}. The novelty of this paper is to propose an automatic variable selection algorithm that makes use of Probabilistic Graphical Models  \cite{jordan2004graphical}, to detect the relevant explicative variables of a linear econometric model on the basis of a large dataset. In this way, we can combine the availability of large datasets with the interpretability of the econometric model,  obtaining models with a god accuracy also in prediction terms.

\bibliographystyle{plain} 

\newpage
\begin{appendices}
\section{Variable Name}

\begin{table}[H]\par\medskip
\centering
\scalebox{0.85}{
\begin{tabular}{|c|c|c|c|}
\hline
\textbf{Label Node} & \textbf{Variable}            & \textbf{Label Node} & \textbf{Variable}              \\ \hline
\textit{1}          & \textit{Community Name}      & \textit{53}         & \textit{PctIlleg}              \\ \hline
\textit{2}          & \textit{fold}                & \textit{54}         & \textit{NumImmig}              \\ \hline
\textit{3}          & \textit{population}          & \textit{55}         & \textit{PctImmigRecent}        \\ \hline
\textit{4}          & \textit{householdsize}       & \textit{56}         & \textit{PctImmigRec5}          \\ \hline
\textit{5}          & \textit{racepctblack}        & \textit{57}         & \textit{PctImmigRec8}          \\ \hline
\textit{6}          & \textit{racePctWhite}        & \textit{58}         & \textit{PctImmigRec10}         \\ \hline
\textit{7}          & \textit{racePctAsian}        & \textit{59}         & \textit{PctRecentImmig}        \\ \hline
\textit{8}          & \textit{racePctHisp}         & \textit{60}         & \textit{PctRecImmig5}          \\ \hline
\textit{9}          & \textit{agePct12t21}         & \textit{61}         & \textit{PctRecImmig8}          \\ \hline
\textit{10}         & \textit{agePct12t29}         & \textit{62}         & \textit{PctRecImmig10}         \\ \hline
\textit{11}         & \textit{agePct16t24}         & \textit{63}         & \textit{PctSpeakEnglOnly}      \\ \hline
\textit{12}         & \textit{agePct65up}          & \textit{64}         & \textit{PctNotSpeakEnglWell}   \\ \hline
\textit{13}         & \textit{numbUrban}           & \textit{65}         & \textit{PctLargHouseFam}       \\ \hline
\textit{14}         & \textit{pctUrban}            & \textit{66}         & \textit{PctLargHouseOccup}     \\ \hline
\textit{15}         & \textit{medIncome}           & \textit{67}         & \textit{PersPerOccupHous}      \\ \hline
\textit{16}         & \textit{pctWWage}            & \textit{68}         & \textit{PersPerOwnOccHous}     \\ \hline
\textit{17}         & \textit{pctWFarmSelf}        & \textit{69}         & \textit{PersPerRentOccHous}    \\ \hline
\textit{18}         & \textit{pctWInvInc}          & \textit{70}         & \textit{PctPersOwnOccup}       \\ \hline
\textit{19}         & \textit{pctWSocSec}          & \textit{71}         & \textit{PctPersDenseHous}      \\ \hline
\textit{20}         & \textit{pctWPubAsst}         & \textit{72}         & \textit{PctHousLess3BR}        \\ \hline
\textit{21}         & \textit{pctWRetire}          & \textit{73}         & \textit{MedNumBR}              \\ \hline
\textit{22}         & \textit{medFamInc}           & \textit{74}         & \textit{HousVacant}            \\ \hline
\textit{23}         & \textit{perCapInc}           & \textit{75}         & \textit{PctHousOccup}          \\ \hline
\textit{24}         & \textit{whitePerCap}         & \textit{76}         & \textit{PctHousOwnOcc}         \\ \hline
\textit{25}         & \textit{blackPerCap}         & \textit{77}         & \textit{PctVacantBoarded}      \\ \hline
\textit{26}         & \textit{indianPerCap}        & \textit{78}         & \textit{PctVacMore6Mos}        \\ \hline
\textit{27}         & \textit{AsianPerCap}         & \textit{79}         & \textit{MedYrHousBuilt}        \\ \hline
\textit{28}         & \textit{OtherPerCap}         & \textit{80}         & \textit{PctHousNoPhone}        \\ \hline
\textit{29}         & \textit{HispPerCap}          & \textit{81}         & \textit{PctWOFullPlumb}        \\ \hline
\textit{30}         & \textit{NumUnderPov}         & \textit{82}         & \textit{OwnOccLowQuart}        \\ \hline
\textit{31}         & \textit{PctPopUnderPov}      & \textit{83}         & \textit{OwnOccMedVal}          \\ \hline
\textit{32}         & \textit{PctLess9thGrade}     & \textit{84}         & \textit{OwnOccHiQuart}         \\ \hline
\textit{33}         & \textit{PctNotHSGrad}        & \textit{85}         & \textit{RentLowQ}              \\ \hline
\textit{34}         & \textit{PctBSorMore}         & \textit{86}         & \textit{RentMedian}            \\ \hline
\textit{35}         & \textit{PctUnemployed}       & \textit{87}         & \textit{RentHighQ}             \\ \hline
\textit{36}         & \textit{PctEmploy}           & \textit{88}         & \textit{MedRent}               \\ \hline
\textit{37}         & \textit{PctEmplManu}         & \textit{89}         & \textit{MedRentPctHousInc}     \\ \hline
\textit{38}         & \textit{PctEmplProfServ}     & \textit{90}         & \textit{MedOwnCostPctInc}      \\ \hline
\textit{39}         & \textit{PctOccupManu}        & \textit{91}         & \textit{MedOwnCostPctIncNoMtg} \\ \hline
\textit{40}         & \textit{PctOccupMgmtProf}    & \textit{92}         & \textit{NumInShelters}         \\ \hline
\textit{41}         & \textit{MalePctDivorce}      & \textit{93}         & \textit{NumStreet}             \\ \hline
\textit{42}         & \textit{MalePctNevMarr}      & \textit{94}         & \textit{PctForeignBorn}        \\ \hline
\textit{43}         & \textit{FemalePctDiv}        & \textit{95}         & \textit{PctBornSameState}      \\ \hline
\textit{44}         & \textit{TotalPctDiv}         & \textit{96}         & \textit{PctSameHouse85}        \\ \hline
\textit{45}         & \textit{PersPerFam}          & \textit{97}         & \textit{PctSameCity85}         \\ \hline
\textit{46}         & \textit{PctFam2Par}          & \textit{98}         & \textit{PctSameState85}        \\ \hline
\textit{47}         & \textit{PctKids2Par}         & \textit{99}         & \textit{LandArea}              \\ \hline
\textit{48}         & \textit{PctYoungKids2Par}    & \textit{100}        & \textit{PopDens}               \\ \hline
\textit{49}         & \textit{PctTeen2Par}         & \textit{101}        & \textit{PctUsePubTrans}        \\ \hline
\textit{50}         & \textit{PctWorkMomYoungKids} & \textit{102}        & \textit{LemasPctOfficDrugUn}   \\ \hline
\textit{51}         & \textit{PctWorkMom}          & \textit{103}        & \textit{ViolentCrimesPerPop}   \\ \hline
\textit{52}         & \textit{NumIlleg}            & \textit{}           & \textit{}                      \\ \hline
\end{tabular}
}
\caption{Name of the variables \textit{Communities and Crime} dataset}
\label{NameVar}
\end{table}
\newpage

\begin{table}[H]\par\medskip
\centering
\scalebox{0.55}{
\begin{tabular}{|c|c|c|c|}
\hline
\textbf{Label Node} & \textbf{Variables}                                                                                        & \textbf{Label Node} & \textbf{Variables}                                                                  \\ \hline
\textit{1}     & \textit{Age dependency ratio (\% of working-age population)}                                              & \textit{68}    & \textit{Population ages 15-19, female (\% of female population)}                    \\ \hline
\textit{2}     & \textit{Age dependency ratio, old (\% of working-age population)}                                         & \textit{69}    & \textit{Population ages 15-19, male (\% of male population)}                        \\ \hline
\textit{3}     & \textit{Age dependency ratio, young (\% of working-age population)}                                       & \textit{70}    & \textit{Population ages 15-64 (\% of total population)}                             \\ \hline
\textit{4}     & \textit{Agriculture, forestry, and fishing, value added (annual \% growth)}                               & \textit{71}    & \textit{Population ages 15-64, female}                                              \\ \hline
\textit{5}     & \textit{Agriculture, forestry, and fishing, value added (constant 2015 US\$)}                             & \textit{72}    & \textit{Population ages 15-64, female (\% of female population)}                    \\ \hline
\textit{6}     & \textit{Agriculture, forestry, and fishing, value added (constant LCU)}                                   & \textit{73}    & \textit{Population ages 15-64, male}                                                \\ \hline
\textit{7}     & \textit{Compulsory education, duration (years)}                                                           & \textit{74}    & \textit{Population ages 15-64, male (\% of male population)}                        \\ \hline
\textit{8}     & \textit{DEC alternative conversion factor (LCU per US\$)}                                                 & \textit{75}    & \textit{Population ages 15-64, total}                                               \\ \hline
\textit{9}     & \textit{Employment to population ratio, 15+, total (\%) (modeled ILO   estimate)}                         & \textit{76}    & \textit{Population ages 20-24, female (\% of female population)}                    \\ \hline
\textit{10}    & \textit{Forest area (\% of land area)}                                                                    & \textit{77}    & \textit{Population ages 20-24, male (\% of male population)}                        \\ \hline
\textit{11}    & \textit{Forest area (sq. km)}                                                                             & \textit{78}    & \textit{Population ages 25-29, female (\% of female population)}                    \\ \hline
\textit{12}    & \textit{GDP (constant 2015 US\$)}                                                                         & \textit{79}    & \textit{Population ages 25-29, male (\% of male population)}                        \\ \hline
\textit{13}    & \textit{GDP (constant LCU)}                                                                               & \textit{80}    & \textit{Population ages 30-34, female (\% of female population)}                    \\ \hline
\textit{14}    & \textit{GDP (current LCU)}                                                                                & \textit{81}    & \textit{Population ages 30-34, male (\% of male population)}                        \\ \hline
\textit{15}    & \textit{GDP (current US\$)}                                                                               & \textit{82}    & \textit{Population ages 35-39, female (\% of female population)}                    \\ \hline
\textit{16}    & \textit{GDP deflator (base year varies by country)}                                                       & \textit{83}    & \textit{Population ages 35-39, male (\% of male population)}                        \\ \hline
\textit{17}    & \textit{GDP growth (annual \%)}                                                                           & \textit{84}    & \textit{Population ages 40-44, female (\% of female population)}                    \\ \hline
\textit{18}    & \textit{GDP per capita (constant 2015 US\$)}                                                              & \textit{85}    & \textit{Population ages 40-44, male (\% of male population)}                        \\ \hline
\textit{19}    & \textit{GDP per capita (constant LCU)}                                                                    & \textit{86}    & \textit{Population ages 45-49, female (\% of female population)}                    \\ \hline
\textit{20}    & \textit{GDP per capita (current LCU)}                                                                     & \textit{87}    & \textit{Population ages 45-49, male (\% of male population)}                        \\ \hline
\textit{21}    & \textit{GDP per capita (current US\$)}                                                                    & \textit{88}    & \textit{Population ages 50-54, female (\% of female population)}                    \\ \hline
\textit{22}    & \textit{GDP per capita growth (annual \%)}                                                                & \textit{89}    & \textit{Population ages 50-54, male (\% of male population)}                        \\ \hline
\textit{23}    & \textit{GDP per capita, PPP (constant 2017 international \$)}                                             & \textit{90}    & \textit{Population ages 55-59, female (\% of female population)}                    \\ \hline
\textit{24}    & \textit{GDP per person employed (constant 2017 PPP \$)}                                                   & \textit{91}    & \textit{Population ages 55-59, male (\% of male population)}                        \\ \hline
\textit{25}    & \textit{GDP, PPP (constant 2017 international \$)}                                                        & \textit{92}    & \textit{Population ages 60-64, female (\% of female population)}                    \\ \hline
\textit{26}    & \textit{GNI (current LCU)}                                                                                & \textit{93}    & \textit{Population ages 60-64, male (\% of male population)}                        \\ \hline
\textit{27}    & \textit{GNI (current US\$)}                                                                               & \textit{94}    & \textit{Population ages 65 and above (\% of total population)}                      \\ \hline
\textit{28}    & \textit{GNI per capita (current LCU)}                                                                     & \textit{95}    & \textit{Population ages 65 and above, female}                                       \\ \hline
\textit{29}    & \textit{GNI per capita, Atlas method (current US\$)}                                                      & \textit{96}    & \textit{Population ages 65 and above, female (\% of female population)}             \\ \hline
\textit{30}    & \textit{GNI, Atlas method (current US\$)}                                                                 & \textit{97}    & \textit{Population ages 65 and above, male}                                         \\ \hline
\textit{31}    & \textit{Human capital index (HCI) (scale 0-1)}                                                            & \textit{98}    & \textit{Population ages 65 and above, male (\% of male population)}                 \\ \hline
\textit{32}    & \textit{Human capital index (HCI), lower bound (scale 0-1)}                                               & \textit{99}    & \textit{Population ages 65 and above, total}                                        \\ \hline
\textit{33}    & \textit{Human capital index (HCI), upper bound (scale 0-1)}                                               & \textit{100}   & \textit{Population ages 65-69, female (\% of female population)}                    \\ \hline
\textit{34}    & \textit{Inflation, GDP deflator (annual \%)}                                                              & \textit{101}   & \textit{Population ages 65-69, male (\% of male population)}                        \\ \hline
\textit{35}    & \textit{Labor force participation rate, total (\% of total population ages 15+)   (modeled ILO estimate)} & \textit{102}   & \textit{Population ages 70-74, female (\% of female population)}                    \\ \hline
\textit{36}    & \textit{Labor force, total}                                                                               & \textit{103}   & \textit{Population ages 70-74, male (\% of male population)}                        \\ \hline
\textit{37}    & \textit{Land area (sq. km)}                                                                               & \textit{104}   & \textit{Population ages 75-79, female (\% of female population)}                    \\ \hline
\textit{38}    & \textit{Lower secondary school starting age (years)}                                                      & \textit{105}   & \textit{Population ages 75-79, male (\% of male population)}                        \\ \hline
\textit{39}    & \textit{Merchandise exports (current US\$)}                                                               & \textit{106}   & \textit{Population ages 80 and above, female (\% of female population)}             \\ \hline
\textit{40}    & \textit{Merchandise imports (current US\$)}                                                               & \textit{107}   & \textit{Population ages 80 and above, male (\% of male population)}                 \\ \hline
\textit{41}    & \textit{Merchandise trade (\% of GDP)}                                                                    & \textit{108}   & \textit{Population density (people per sq. km of land area)}                        \\ \hline
\textit{42}    & \textit{Military expenditure (\% of general government expenditure)}                                      & \textit{109}   & \textit{Population growth (annual \%)}                                              \\ \hline
\textit{43}    & \textit{Military expenditure (current USD)}                                                               & \textit{110}   & \textit{Population, female}                                                         \\ \hline
\textit{44}    & \textit{Official exchange rate (LCU per US\$, period average)}                                            & \textit{111}   & \textit{Population, female (\% of total population)}                                \\ \hline
\textit{45}    & \textit{People practicing open defecation (\% of population)}                                             & \textit{112}   & \textit{Population, male}                                                           \\ \hline
\textit{46}    & \textit{People practicing open defecation, rural (\% of rural population)}                                & \textit{113}   & \textit{Population, male (\% of total population)}                                  \\ \hline
\textit{47}    & \textit{People practicing open defecation, urban (\% of urban population)}                                & \textit{114}   & \textit{Population, total}                                                          \\ \hline
\textit{48}    & \textit{People using at least basic drinking water services (\% of   population)}                         & \textit{115}   & \textit{PPP conversion factor, GDP (LCU per international \$)}                      \\ \hline
\textit{49}    & \textit{People using at least basic drinking water services, rural (\% of rural   population)}            & \textit{116}   & \textit{Preprimary education, duration (years)}                                     \\ \hline
\textit{50}    & \textit{People using at least basic drinking water services, urban (\% of urban   population)}            & \textit{117}   & \textit{Price level ratio of PPP conversion factor (GDP) to market exchange   rate} \\ \hline
\textit{51}    & \textit{People using at least basic sanitation services (\% of population)}                               & \textit{118}   & \textit{Primary education, duration (years)}                                        \\ \hline
\textit{52}    & \textit{People using at least basic sanitation services, rural (\% of rural   population)}                & \textit{119}   & \textit{Primary school starting age (years)}                                        \\ \hline
\textit{53}    & \textit{People using at least basic sanitation services, urban (\% of urban   population)}                & \textit{120}   & \textit{Proportion of seats held by women in national parliaments (\%)}             \\ \hline
\textit{54}    & \textit{Personal remittances, received (\% of GDP)}                                                       & \textit{121}   & \textit{Refugee population by country or territory of asylum}                       \\ \hline
\textit{55}    & \textit{Personal remittances, received (current US\$)}                                                    & \textit{122}   & \textit{Refugee population by country or territory of origin}                       \\ \hline
\textit{56}    & \textit{Population ages 00-04, female (\% of female population)}                                          & \textit{123}   & \textit{Rural population}                                                           \\ \hline
\textit{57}    & \textit{Population ages 00-04, male (\% of male population)}                                              & \textit{124}   & \textit{Rural population (\% of total population)}                                  \\ \hline
\textit{58}    & \textit{Population ages 0-14 (\% of total population)}                                                    & \textit{125}   & \textit{Rural population growth (annual \%)}                                        \\ \hline
\textit{59}    & \textit{Population ages 0-14, female}                                                                     & \textit{126}   & \textit{Secondary education, duration (years)}                                      \\ \hline
\textit{60}    & \textit{Population ages 0-14, female (\% of female population)}                                           & \textit{127}   & \textit{Secure Internet servers}                                                    \\ \hline
\textit{61}    & \textit{Population ages 0-14, male {[}SP.POP.0014.MA.IN{]}}                                               & \textit{128}   & \textit{Secure Internet servers (per 1 million people)}                             \\ \hline
\textit{62}    & \textit{Population ages 0-14, male (\% of male population)}                                               & \textit{129}   & \textit{Unemployment, total (\% of total labor force) (modeled ILO estimate)}       \\ \hline
\textit{63}    & \textit{Population ages 0-14, total}                                                                      & \textit{130}   & \textit{Urban population}                                                           \\ \hline
\textit{64}    & \textit{Population ages 05-09, female (\% of female population)}                                          & \textit{131}   & \textit{Urban population (\% of total population)}                                  \\ \hline
\textit{65}    & \textit{Population ages 05-09, male (\% of male population)}                                              & \textit{132}   & \textit{Urban population growth (annual \%)}                                        \\ \hline
\textit{66}    & \textit{Population ages 10-14, female (\% of female population)}                                          & \textit{133}   & \textit{Women Business and the Law Index Score (scale 1-100)}                       \\ \hline
\textit{67}    & \textit{Population ages 10-14, male (\% of male population)}                                              & \textit{}      & \textit{}                                                                           \\ \hline
\end{tabular}
\caption{Name of the variables \textit{World Development Indicators} dataset}
\label{NameVar1}
}

\end{table}
\section{Availability}
The analyses were performed using the R library
gRapHD which the authors have made available to the R community
via the CRAN repository \cite{de2009high} (de Abreu GCG,
Labouriau R, Edwards D: High-dimensional Graphical
Model Search with gRapHD R package, submitted to J.
Stat. Software).

\end{appendices}
\end{document}